\documentclass[10pt]{article}
\usepackage{graphicx}
\graphicspath{{C:/Users/whsigris/Dropbox/HSLU/Projects/MixedBoost/plots_presentation/}}
\usepackage{natbib}
\usepackage{url} 
\usepackage{amssymb,amsmath,amsthm,graphicx,bbm}
\usepackage[linesnumbered,ruled]{algorithm2e}
\usepackage{algorithmic}
\usepackage[toc,page]{appendix}

\DeclareMathOperator*{\argmin}{argmin}

\newtheorem{theorem}{Theorem}[section]

\newtheorem{proposition}[theorem]{Proposition}

\newcommand{\blind}{0}

\begin{document}

\def\spacingset#1{\renewcommand{\baselinestretch}%
{#1}\small\normalsize} \spacingset{1}

\if0\blind
{
  \title{\bf Gaussian Process Boosting}
  
  \author{Fabio Sigrist\thanks{Email: fabio.sigrist@hslu.ch. Address: Lucerne University of Applied Sciences and Arts, Suurstoffi 1, 6343 Rotkreuz, Switzerland.}\\
   Lucerne University of Applied Sciences and Arts}
\date{}
  \maketitle
}
\fi

\bigskip

\spacingset{1}

\providecommand{\keywords}[1]
{
	\small
	\textbf{\textit{Keywords:}} #1
	\normalsize
}

\begin{abstract}%
We introduce a novel way to combine boosting with Gaussian process and mixed effects models. This allows for relaxing, first, the zero or linearity assumption for the prior mean function in Gaussian process and grouped random effects models in a flexible non-parametric way and, second, the independence assumption made in most boosting algorithms. The former is advantageous for prediction accuracy and for avoiding model misspecifications. The latter is important for efficient learning of the fixed effects predictor function and for obtaining probabilistic predictions. Our proposed algorithm is also a novel solution for handling high-cardinality categorical variables in tree-boosting. In addition, we present an extension that scales to large data using a Vecchia approximation for the Gaussian process model relying on novel results for covariance parameter inference. We obtain increased prediction accuracy compared to existing approaches on multiple simulated and real-world data sets.
\end{abstract}


\section{Introduction}\label{intro}
Boosting \citep{freund1996experiments, friedman2001greedy} is a machine learning technique that often achieves superior prediction accuracy on tabular data sets \citep{nielsen2016tree, shwartz2021tabular, grinsztajn2022tree}. This is reflected in statements such as ``[i]n general `boosted decision trees' is regarded as the most effective off-the-shelf nonlinear learning method for a wide range of application problems"  \citep{johnson2013learning}. Apart from this, the wide adoption of tree-boosting in applied machine learning and data science is due to several advantages: boosting with trees as base learners can automatically account for complex non-linearities, discontinuities, and high-order interactions, it is robust to outliers in and multicollinearity among predictor variables, it is scale-invariant to monotone transformations of the predictor variables, it can handle missing values in predictor variables automatically by loosing almost no information \citep{elith2008working}, and it can perform variable selection. 

In boosting and in many other state-of-the-art supervised machine learning algorithms, a flexible and potentially complex function relates a set of predictor variables to a response variable. However, the response variable is usually assumed to be independent across observations conditional on the predictor variables. This means that potential residual correlation, i.e., correlation that is not accounted for by the predictor function, is ignored. As we show in our experiments, modeling such correlation allows for more efficient learning of the predictor function, and it is important for prediction accuracy, in particular for probabilistic predictions and for predicting sums (or averages) over, e.g., space. Example applications of the latter include the prediction of global average temperatures, the total rainfall in a catchment area, or the total value of a portfolio of real estate objects. Apart from the potentially unrealistic independence assumption, tree-boosting can have difficulty with high-cardinality categorical variables and predictions are discontinuous. The latter is often unrealistic for spatial and spatio-temporal data.

Gaussian processes \citep{williams2006gaussian} are flexible non-parametric function models that achieve state-of-the-art prediction accuracy and allow for making probabilistic predictions \citep{gneiting2007probabilistic}. They are used in areas such as (i) non-parametric regression, (ii) modeling of time series, spatial, and spatio-temporal data \citep{shumway2017time, banerjee2014hierarchical, cressie2015statistics}, (iii) emulation of large computer experiments \citep{kennedy2001bayesian}, (iv) optimization of expensive black-box functions \citep{jones1998efficient}, (v) parameter tuning in machine learning models \citep{snoek2012practical}, and others. Further, mixed effects models \citep{laird1982random, pinheiro2006mixed} with grouped, or clustered, random effects are widely used in various scientific disciplines for modeling data with a grouping structure such as panel and longitudinal data. 

In Gaussian process and mixed effects models, the prior mean is traditionally assumed to be either zero or a linear function of predictor variables. Residual structured variation is then modeled using a zero-mean Gaussian process and/or a grouped random effects model. However, both the zero-mean and the linearity assumption can be unrealistic, and higher prediction accuracy can be obtained by relaxing these assumptions; see, e.g., our experiments in Sections \ref{simul} and \ref{data_appl}. Furthermore, if the prior mean function of a Gaussian process model is misspecified, spurious second-order non-stationarity can occur as the covariance function of a misspecified model equals the true covariance function plus the squared bias of the mean function \citep{fuglstad2015does, schmidt2020flexible}. It is thus important to first correctly model the prior mean function before accounting for potential residual second-order non-stationarity. 


In this article, we propose a novel way to combine boosting with Gaussian process and grouped random effects models to remedy the above-mentioned drawbacks and leverage the advantages of both approaches. In particular, the goal is to relax, first, the linearity or zero prior mean assumption in Gaussian process and mixed effects models and, second, the independence assumption in boosting. This is done by considering a mixed effects model where the prior mean fixed effects function, also denoted as predictor function in the following, is assumed to be a non-parametric function of predictor variables, and the random effects can consist of various combinations of Gaussian processes and grouped random effects. We propose to model the non-linear predictor function by an ensemble of base learners, such as regression trees \citep{breiman1984classification}, learned in a stage-wise manner using boosting, and the parameters of the covariance structure of the random effects are jointly estimated with the predictor function. As every categorical variable corresponds to a grouping of the data, grouped random effects models can be seen as a tool for modeling categorical variables. In light of this, we also present a novel approach for dealing with high-cardinality categorical predictor variables in tree-boosting.

\subsection{Relation to existing work}\label{exist_work}
The majority of existing Gaussian process and mixed effects models assume that the prior mean fixed effects function is linear or zero. Little research has been done on combining modern supervised machine learning techniques, such as boosting or random forest, with mixed effects models and Gaussian processes. In the following, we give a brief review of existing literature for mixed effects models where a non-linear predictor function is estimated in a flexible, non-parametric way.

To relax the linearity assumption in mixed effects models, \citet{tutz2007boosting} and \citet{groll2012regularization} propose to use generalized additive models (GAMs) \citep{Hastie1986, wood2017generalized}. However, the structure of the predictor function has to be determined a priori by specifying, for instance, main and second-order interaction effects. In general, this can thus result in model misspecification. For the special case of grouped random effects models for clustered or longitudinal data, several non-parametric, machine learning-based approaches have been proposed. This includes \citet{hajjem2011mixed}, \citet{sela2012re}, and \citet{fu2015unbiased} which use regression trees for the predictor function, and \citet{hajjem2014mixed} which use random forest to model the predictor function. 
Both the MERT and MERF algorithms of \citet{hajjem2011mixed, hajjem2014mixed} and the RE-EM tree algorithm of \citet{sela2012re} repeatedly (i) learn the predictor function using a machine learning technique, (ii) estimate covariance parameters, and (iii) calculate predictions for random effects. This can make such approaches computationally demanding, in particular when the sample size is large and the predictor function is modeled using a complex model because both covariance parameters and the fixed effects functions are repeatedly learned in every iteration. Since it is not clearly stated in the previous literature which optimization problems are solved by these algorithms and how they are related to the EM algorithm for mixed effects models \citep{laird1982random}, we elaborate on this in Appendix \ref{MERFalgo}. \citet{pande2017boosted} propose a tree-boosting approach called boostmtree for a special type of mixed effects models for longitudinal data with single-level grouped random effects and where predictor variables are assumed to be constant within subjects. Focusing on modeling complex interactions between time and predictor variables, their approach differs from ours in several directions. For instance, they do not cover Gaussian processes or other forms of random effects models with complex clustering such as crossed or nested random effects. Furthermore, they (re-)estimate the covariance parameters in every boosting iteration using the \texttt{nlme} R package, and their boostmtree algorithm uses a certain form of ``in-sample cross-validation" to avoid overfitting. Table \ref{compare_models} provides a brief overview of the above-mentioned methods for linear and non-linear mixed effects models.


\begin{table}[ht!]
\centering
\begingroup
\scalebox{0.9}{
\begin{tabular}{lll}
\hline
\hline
\bf Method &  \bf Fixed effects $F(\cdot)$  & \bf Random effects $Zb$ \\ 
\hline
LMM & Linear & Grouped random effects\\
\hline
Gaussian processes (GP) & Linear or $F(\cdot)\equiv 0$ & GP\\
\hline
 \begin{tabular}{@{}l@{}} Grouped random effects \\ with GAM fixed effects \end{tabular} & Non-linear, GAM & Grouped random effects\\
\hline
Classical boosting & Non-linear, boosting & None\\
\hline
MERT & Non-linear, regression tree & Grouped random effects \\
\hline
RE-EM tree & Non-linear, regression tree & Grouped random effects \\
\hline
MERF & Non-linear, random forest & Grouped random effects \\
\hline
boostmtree & \begin{tabular}{@{}l@{}}Non-linear, boosting, \\predictor variables are \\ constant within groups \end{tabular} & \begin{tabular}{@{}l@{}} Special type of grouped \\ random effects for \\ longitudinal data \end{tabular} \\
\hline
GPBoost & Non-linear, boosting &  \begin{tabular}{@{}l@{}} GP and grouped \\ random effects \end{tabular}  \\
\hline
\hline
\end{tabular}
}
\endgroup
\caption{Short overview of methods for linear and non-linear mixed effects models.} 
\label{compare_models}
\end{table}


A straightforward approach to applying machine learning methods to spatial and grouped data consists of simply including spatial locations as continuous variables and grouping variables as categorical variables in the set of predictor variables for the predictor function. For linear models, this is equivalent to using fixed effects instead of random effects. However, this approach has several drawbacks. For instance, when modeling spatial data, it is often required that the spatial effect is continuous over space, but tree-boosting and random forest produce discontinuous functions. A way to avoid this problem in boosting is to use base learners that are continuous in the locations. This is the approach proposed in \citet{hothorn2010model} where splines are used to model spatial effects and ridge regression is used to model grouped effects. As we argue in the following, this approach still has some weaknesses. First, all boosting approaches which model spatial and grouped effects using deterministic base learners have the drawback that they assume a deterministic relationship and that the residual error term is the only source of stochastic variation. In contrast, Gaussian process and mixed effects models are probabilistic models that explicitly account for correlations and provide probabilistic predictions. Modeling correlation is particularly important for probabilistic predictions of, e.g., areal sums as correlation between the different random effects needs to be taken into account; see our experiments in Sections \ref{simul} and \ref{data_appl}. Besides, estimates are usually less efficient when using fixed effects instead of random effects in linear models. It is thus likely that the predictor function is also less efficiently estimated in such an independent ``fixed-effects" boosting approach. Our simulated experiments in Section \ref{simul} support this hypothesis.  Moreover, splines have the disadvantage that they suffer from the so-called curse of dimensionality in non-spatial applications when the dimension of the ``locations", or features, is large, and the locations are thus sparse in space; see Section \ref{intro} for examples where this occurs.

The above-mentioned efficiency problem of fixed effects models is related to the fact that classical tree-boosting algorithms can have difficulties with high-cardinality categorical variables. Concerning this, the approach of the \texttt{LightGBM} boosting library \citep{ke2017lightgbm} works by partitioning all groups into two subsets when finding splits in the tree-building algorithm. Since there are $2^{m-1}$ possible partitions of $m$ different groups, an efficient approach based on \citet{fisher1958grouping} is used in \texttt{LightGBM}. Further, \citet{CatBoost2017} present an approach based on ordered target statistics calculated using random partitions of the training data for handling categorical predictor variables.



Finally, we also mention \citet{song2005new} who use Gaussian processes as base learners in special types of boosting algorithms, and \citet{sigrist2019KTBoost} who combines reproducing kernel Hilbert space (RKHS) regression functions with trees as base learners in boosting algorithms. However, these approaches using Gaussian processes or their deterministic counterparts, kernel machines, as base learners assume conditional independence among observations and are thus different from our approach.

\section{A Non-linear and Non-parametric Mixed Effects Model}\label{model_def}
\subsection{Model assumptions and notation}\label{model_assum}
Following the terminology used in the mixed effects models literature \citep{laird1982random, pinheiro2006mixed}, we assume a mixed effects model of the form
\begin{equation}\label{modunit}
y=F(X)+Zb+\epsilon,~~b\sim \mathcal{N}(0,\Sigma),~~\epsilon \sim \mathcal{N}(0,\sigma^2I_{n}),
\end{equation}
where $y=(y_{1},\dots,y_{n})^T\in\mathbb{R}^{n}$ is the response variable, $F(X)\in\mathbb{R}^{n}$ are called fixed effects, $b\in\mathbb{R}^{m}$ are zero-mean random effects with covariance matrix $\Sigma\in \mathbb{R}^{m\times m}$, and $\epsilon=(\epsilon_1,\dots,\epsilon_n)^T\in\mathbb{R}^{n}$ is an independent error term. Specifically, $F(X)$ is the row-wise evaluation of a function $F:\mathbb{R}^p\rightarrow\mathbb{R}$, $F(X)=(F(X_1),\dots,F(X_n))^T$, where $X_{i}=(X_{i1}\dots,X_{ip})^T\in\mathbb{R}^{p}$ is the $i$-th row of $X$ containing predictor variables for observation $i$, $i=1,\dots,n$. The matrices $X\in\mathbb{R}^{n\times p}$ and $Z\in\mathbb{R}^{n\times m}$ are the fixed and random effects predictor variable matrices. Further, $n$ denotes the number of data points, $m$ denotes the dimension of the random effects $b$, and $p$ denotes the number of predictor variables in $X$. 

The random effects vector $b$ is either a finite-dimensional version of a Gaussian process and/or it can contain grouped random effects.\footnote{We assume that the random effects follow a normal distribution, but moderate violations of this assumption have been shown to have only a small effect on prediction accuracy in the context of generalized linear mixed models \citep{mcculloch2011misspecifying}.} The covariance matrix $\text{Cov}(b)=\Sigma$ is usually assumed to be parametrized by a relatively low number of parameters, and it can depend on predictor variables $S\in \mathbb{R}^{n\times d}$. For instance, the predictor variables $S$ can consist of locations and time points for spatial and time series data, or they can be any input features in machine learning applications. For notational simplicity, we suppress the dependence of $\Sigma$ on its parameters and also on $S$. The matrix $Z$ relates the random effects $b$ to the response variable $y$ and it is often an incidence matrix with entries in $\{0,1\}$. In the case of grouped random effects, the columns of $Z$ correspond to dummy variables, also called one-hot encoded variables, of the categorical variables that define the grouping structure. For Gaussian processes, $Z$ is usually simply an identity matrix, but it can also be a binary incidence matrix to model multiple observations at the same locations. Further, $Z$ can also contain covariate data, e.g., in the case of random coefficient models \citep{gelfand2003spatial, pinheiro2006mixed} also called spatially varying coefficient models for spatial data. The predictor variables in $Z$ and $S$ may or may not be a subset of the predictor variables in $X$, and vice versa. In Section \ref{examples_mod}, we outline several examples and special cases of models for $Zb$. Note that conditional on $F(X)$ and $Z$, dependence among the response variable $y$ can arise either due to the matrix $Z$ being non-diagonal or due to the covariance matrix $\Sigma$ of the random effects being non-diagonal.
 
In linear mixed effects models (LMMs) and linear Gaussian process models in spatial statistics, it is assumed that $F(X)=X^T\beta$, where $\beta\in \mathbb{R}^p$ is a vector of coefficients. We denote models with $F(X)=X^T\beta$ as linear models in this article. Further, in the machine learning literature on Gaussian processes, is usually assumed that the prior mean function is zero, $F(X)\equiv 0$. As mentioned above, one of the novel contributions of our approach is to relax this linearity or zero prior mean assumption in a flexible and non-parametric manner.

Concerning $F(\cdot)$, we assume that $F(\cdot)$ is a function in a function space $\mathcal{H}$ which is the linear span of a set $\mathcal{S}$ of so-called base learners $f_j(\cdot):\mathbb{R}^p\rightarrow \mathbb{R}$. Possible base learners are linear functions \citep{buehlmann2006boosting}, smoothing splines \citep{buhlmann2003boosting}, wavelets \citep{saberian2011taylorboost}, reproducing kernel Hilbert space (RKHS) regression functions \citep{sigrist2019KTBoost}, neural networks \citep{schwenk2000boosting}, and regression trees \citep{breiman1984classification}, with the latter being the most popular choice, in particular, in applied machine learning due to the advantages listed in Section \ref{intro}. For defining functional derivatives, we additionally assume that the space $\mathcal{H}$ is normed. For instance, assuming that the $X_i$'s are identically distributed and that all $F\in\mathcal{H}$ are square integrable with respect to the law of $X_1$, a norm on $\mathcal{H}$ can defined by the inner product $\langle F,G\rangle = E_{X_1}(F(X_1)G(X_1))$ for $F,G\in \mathcal{H}$. 

The marginal distribution of $y$ in \eqref{modunit} is 
\begin{equation*}
y \sim \mathcal{N}\left(F(X),\Psi\right), ~~~~\Psi=Z\Sigma  Z^T+ \sigma^2 I_{n},
\end{equation*}
and the negative log-likelihood of this model is given by
\begin{equation}\label{loss_def}
L(y,F,\theta)=\frac{1}{2}(y-F)^T{\Psi}^{-1}(y-F)+\frac{1}{2}\log\det\left(\Psi\right)+\frac{n}{2}\log(2\pi),
\end{equation}
where we abbreviate $F=F(X)$ and $\theta\in\Theta\subset \mathbb{R}^q$ denotes all variance and covariance parameters, i.e., $\sigma^2$ and the parameters of $\Sigma$. In this article, we assume that the first element of $\theta$ is the error variance, $\theta_1=\sigma^2$. Further, to distinguish a function from its evaluation, we use the symbols ``$F(\cdot)$" to denote a function and ``$F$" for the function evaluated at $X$. 

Factoring out the error variance $\sigma^2$ by setting 
$$\Sigma^\dagger=\Sigma/\sigma^2 ~~~~ \text{and} ~~~~ \Psi^\dagger=\Psi/\sigma^2$$
gives the equivalent form
\begin{equation*}
y \sim \mathcal{N}\left(F(X),\sigma^2\Psi^\dagger\right).
\end{equation*}
As this leads to a closed form expression for estimating $\sigma^2$ (see Section \ref{GPB}), it can be beneficial to use a reparametrization in which, after factoring out $\sigma^2$, all other variance parameters are replaced by the ratio of their original value and the error variance $\sigma^2$, and $\Psi^\dagger$ does not depend on $\sigma^2$. Calculations in the \texttt{GPBoost} library (see Section \ref{software}), which implements our methodology, are based on this reparametrization.


\subsection{Examples and special cases of random effects models}\label{examples_mod}
In this section, we outline several examples of random effects models. We distinguish between the following broad classes of random effects models: (i) grouped random effects models, (ii) Gaussian process models, and (iii) combinations of these two types of random effects. 

\subsubsection{Grouped random effects model}
Grouped, or clustered, random effects models are applied to grouped data. For instance, a grouping of data can occur if there are several units each having multiple observations which are potentially dependent within units. Grouped random effects then account for correlation due to this grouping structure. Depending on the application area, such grouped data is also denoted as longitudinal, panel, or repeated measurement data if different units are repeatedly measured over time.

In a single-level grouped random effects model, there is a single grouping of the data into $m$ different groups. For every group $j$, $j=1,\dots,m$, there is a random effect $b_j\in\mathbb{R}$, and the random effects $b_j$ are assumed to be independent and identically distributed with $\text{Cov}(b)=\Sigma=\sigma_1^2 I_m$. We thus have
$$\Psi=\sigma_1^2 Z Z^T+ \sigma^2I_{n},$$
where the matrix $Z\in \{0,1\}^{n\times m}$ is an incidence matrix that relates group-level
random effects to observations. Such a single-level model can be easily extended to allow for multiple random effects which are crossed and hierarchically nested and can also consist of random coefficients (random slopes). In the latter case, $Z$ is no longer a binary incidence matrix but contains covariate data.

Note that every categorical variable of cardinality $m$ corresponds to a grouping of the data into $m$ different groups, and the columns of $Z$ correspond to dummy, or one-hot encoded, variables of the categorical variable representing the grouping structure. In general, grouped random effects models can thus be used to model (high-cardinality) categorical predictor variables.

\subsubsection{Gaussian process model}\label{gaussproc}
In a Gaussian process model, the random effects $b=(b(s_1),\dots,b(s_m))^T$ are a finite-dimensional version of a Gaussian process $b(s)$ with a covariance function
$$Cov(b(s),b(s'))=c(s,s'),~~s,s'\in \mathbb{R}^{d},$$
observed at locations $s_1,\dots,s_m$. Note that we use the terminology ``locations" adopted in spatial statistics, but the locations can in general consist of non-spatial predictor variables also called input features in machine learning. In particular, the locations can also contain the predictor variables $X_i$ or a subset of them.

Often, the covariance function is assumed to be second-order stationary, mean-square continuous, and parametrized of the form
$$c(s,s')=\sigma_1^2 r(\|s-s'\|/\rho),$$
where $r$ is an isotropic autocorrelation function with $r(0)=1$, $\sigma_1^2=Var(b(s))$, and $\rho$ is a so-called range parameter that determines how fast the autocorrelation decays with distance. Examples of autocorrelation functions include the exponential function $r(\|s-s'\|/\rho)=\exp(-\|s-s'\|/\rho)$ and the Gaussian function $r(\|s-s'\|/\rho)=\exp\left(-\left(\|s-s'\|/\rho\right)^2\right)$. The extension to more general covariance functions and also to multivariate Gaussian processes is straightforward, and our methodology presented in Section \ref{estimation} does not rely on stationarity assumptions. For a stationary Gaussian process, we obtain the following covariance matrix
$$\Psi=Z\Sigma Z^T+ \sigma^2I_{n},$$
where $\Sigma\in \mathbb{R}^{m\times m}$ has entries
$$\left(\Sigma\right)_{jk}=\sigma_1^2 r(d_{jk}/\rho)$$
and $d_{jk}=\|s_j-s_k\|$, $j,k=1,\dots,m$.

\subsubsection{Joint grouped random effects and Gaussian process models}
Grouped random effects and Gaussian processes can also be combined. For instance, such models are used for repeated measures and longitudinal data. One can assume that within groups, there is temporal and/or spatial dependence modeled by a Gaussian process. In a single-level grouped random effect model, this means that every group contains a Gaussian process and the different Gaussian processes are independent among each other. Alternatively, one can assume that there is a single global Gaussian process in combination with grouped random effects, i.e. the same Gaussian process is related to all observations and it accounts, for instance, for spatial or temporal correlation among all observations. 

\section{Combining Gaussian Process and Mixed Effects Models with Boosting}\label{estimation}
Our goal is to find the joint minimizer
\begin{equation}\label{optim_def}
(\hat F(\cdot), \hat \theta) =\argmin_{(F(\cdot),\theta) \in (\mathcal{H},\Theta)}R(F(\cdot),\theta),
\end{equation}
where $R(F(\cdot),\theta)$ is a risk functional defined as
\begin{equation}\label{obj_func}
R(F(\cdot),\theta): ~~(F(\cdot), \theta) ~\mapsto~ L(y,F,\theta)\Big|_{F=F(X)},
\end{equation}
and $L(y,F,\theta)$ is a loss function corresponding to the negative log-likelihood given in \eqref{loss_def}. Note that $R(F(\cdot),\theta)$ is determined by evaluating $F(\cdot)$ at $X$ and then calculating $L(y,F=F(X),\theta)$. We recall that $F(\cdot)$ is a function in a function space $\mathcal{H}$ and $\theta\in\Theta\subset \mathbb{R}^q$ is a vector of all variance and covariance parameters. The risk functional $R(F(\cdot),\theta)$ is, in general, infinite-dimensional in its first argument and finite-dimensional in its second argument. We propose to do the minimization in \eqref{optim_def} using a boosting algorithm presented in the following.

\subsection{Boosting for $\theta$ fixed}\label{mult_gauss_loss}
We first show how boosting \citep{freund1996experiments, breiman1998arcing,friedman2000additive,mason2000boosting, friedman2001greedy, buhlmann2007boosting} works in our case when the covariance parameters $\theta$ are given. For fixed $\theta$, boosting finds a minimizer of the empirical risk functional $R(F(\cdot),\theta)$ in a stagewise manner by sequentially adding an update $f_m(\cdot)$ to the current estimate $F_{m-1}(\cdot)$:
\begin{equation}\label{boostupdate}
F_m(\cdot)= F_{m-1}(\cdot)+ f_m(\cdot),~~f_m\in \mathcal{S}, ~~m=1,\dots,M,
\end{equation}
where $f_m(\cdot)$ is chosen in a way such that its addition results in the minimization of the risk. This minimization can usually not be done analytically and, consequently, an approximation is used. In most boosting algorithms, such an approximation is obtained using either a penalized functional first-order or a functional second-order Taylor expansion of the risk functional around the current estimate $F_{m-1}(\cdot)$. This then corresponds to functional gradient descent or functional Newton steps. It is also possible to combine gradient and Newton steps by first learning part of the parameters of the base learners using a gradient step and the remaining part using a Newton update \citep{friedman2001greedy}. See \citet{sigrist2018gradient} for more information on the distinction between gradient, Newton, and hybrid gradient-Newton boosting. In contrast to classical boosting algorithms, the risk functional in \eqref{obj_func} is not a sum of $n$ independent terms as all observations $y$ are dependent conditional on $X$. In other words, we generally have only one independent multivariate sample instead of $n$ independent samples.

In gradient boosting, $f_m(\cdot)$ is found as the least squares approximation to the vector obtained when evaluating the negative functional gradient, i.e., the negative G\^ateaux derivative, of $R(F(\cdot),\theta)$ at $(F_{m-1}(\cdot), I_{X_i}(\cdot))$, $i=1,\dots,n$, where $I_{X_i}(\cdot)$ are indicator functions which equal $1$ at $X_i$ and $0$ otherwise. Equivalently, one can show that $f_m(\cdot)$ corresponds to a minimizer of a first-order functional Taylor approximation of $R(F_{m-1}(\cdot)+f(\cdot),\theta)$ around $F_{m-1}(\cdot)$ with an $L^2$ penalty on $f(\cdot)$ evaluated at $(X_i)$, $i=1,\dots,n$; see, e.g., \citet{sigrist2018gradient} for more information. It is easily seen that the negative G\^ateaux derivative of $R(F(\cdot),\theta)$ evaluated at $(F_{m-1}(\cdot), I_{X_i}(\cdot))$, $i=1,\dots,n$, is given by
$$-\frac{\partial}{\partial F}L(y,F,\theta)\Big|_{F=F_{m-1}}=\Psi^{-1}(y-F_{m-1}).$$ 
This means that for $\theta$ fixed, gradient boosting finds $f_m(\cdot)$ as the least squares approximation
\begin{equation}\label{grad_boost}
f_m(\cdot)=\argmin_{f(\cdot)\in \mathcal{S}} \left\|\Psi^{-1}(y-F_{m-1})-f\right\|^2,
\end{equation}
where $f=(f(X_1),\dots,f(X_n))^T$. 

In Newton boosting, $f_m(\cdot)$ is found as the minimizer of a second-order functional Taylor approximation to  $R(F_{m-1}(\cdot)+f(\cdot),\theta)$ around $F_{m-1}(\cdot)$. In our case, this gives
\begin{equation}\label{newt_boost}
f_m(\cdot)=\argmin_{f(\cdot)\in \mathcal{S}}\left(y-F_{m-1}-f\right)^T\Psi^{-1}\left(y-F_{m-1}-f\right).
\end{equation}

As mentioned, there also exists a hybrid gradient-Newton boosting version which learns part of the parameters of the base learner $f_m(\cdot)$ using a gradient step and the remaining part using a Newton step. In the following paragraph, we assume that the base learners can be written in the form
$$f(\cdot)=h(\cdot;\alpha)^T\gamma, ~~ h(\cdot;\alpha),\gamma\in \mathbb{R}^K, ~~\alpha\in \mathbb{R}^{Q},$$
where $\alpha$ and $\gamma$ denote parameters of the base learners and $h(\cdot;\alpha):\mathbb{R}^p\rightarrow \mathbb{R}^K$. For instance, for regression trees, $\alpha$ encodes split locations and variables, $\gamma$ contains the terminal node values, and $h(\cdot;\alpha)$ is a function that maps predictor variables to terminal tree nodes. In this case, hybrid gradient-Newton boosting first learns $\hat\alpha$ and $\hat\gamma$ using a gradient boosting step given in \eqref{grad_boost} and sets $\alpha_m=\hat\alpha$, and then updates $\gamma_m=\hat\gamma$ using a Newton boosting step defined in \eqref{newt_boost} with $\alpha_m$ fixed. For the latter Newton step, an explicit generalized least squares solution is obtained as
$$\gamma_m =\left(h_{\alpha_m}^T \Psi^{-1} h_{\alpha_m}\right)^{-1} h_{\alpha_m}^T\Psi^{-1}(y-F_{m-1}),$$
where $h_{\alpha_m}\in\mathbb{R}^{n\times K}$ is the matrix with entries $(h_{\alpha_m})_{ik}=h(X_i;\alpha_m)_k$, $i=1,\dots,n$, $k=1,\dots,K$. 

We note that in Newton boosting as well as hybrid gradient-Newton boosting, the norm of $f_m(\cdot)$, i.e., the step-length, does not depend on the scaling of the loss function, in particular not on $\sigma^2$, because $$\argmin_{f(\cdot)\in \mathcal{S}}\left(y-F_{m-1}-f\right)^T\Psi^{-1}\left(y-F_{m-1}-f\right) = \argmin_{f(\cdot)\in \mathcal{S}}\left(y-F_{m-1}-f\right)^T\left(\Psi^\dagger\right)^{-1}\left(y-F_{m-1}-f\right),$$
and $\Psi^\dagger$, which is defined in Section \ref{model_assum}, does not depend on $\sigma^2$. Since gradients are scale-dependent, this does not hold true for gradient boosting.

It has been empirically observed \citep{friedman2001greedy} that higher prediction accuracy can be obtained by damping the update in \eqref{boostupdate}:
\begin{equation}\label{damping}
F_m(\cdot)= F_{m-1}(\cdot)+ \nu f_m(\cdot),~~\nu>0,
\end{equation}
where $\nu$ is called the shrinkage parameter or learning rate. Further, as in finite-dimensional optimization, functional gradient descent can be accelerated using momentum. For instance, \citet{biau2019accelerated} and \citet{lu2019accelerating} propose to use Nesterov acceleration \citep{nesterov2004introductory} for gradient boosting.

\subsection{Gaussian process boosting}\label{GPB}
A straightforward approach for finding a joint minimizer in \eqref{optim_def} would consist of iteratively first doing one approximate functional gradient or Newton descent step for $F(\cdot)$, and then performing one first- or second-order optimization step for $\theta$. Despite being attractive from a computational point of view, this has the following drawback. For finite samples, boosting tends to overfit, in particular for regression, and early stopping has to be applied as a form of regularization to prevent this. However, there is no guarantee that $\theta$ has converged to a minimum when early stopping is applied after a certain number of iterations. A possible solution to avoid this problem consists of doing coordinate descent, also called block descent, for both $F(\cdot)$ and $\theta$, i.e., iteratively doing full optimization in both directions. But this has the drawback that it is computationally expensive as both $F(\cdot)$ and the covariance parameters $\theta$ need to be repeatedly learned. 

%

Our proposed solution presented in Algorithm \ref{gpboost_algo} is to combine functional gradient or Newton boosting steps for $F(\cdot)$ with coordinate descent steps for $\theta$. Specifically, in every iteration, we first determine $\theta_m=\argmin_{\theta\in\Theta}L(y,F_{m-1},\theta)$ and then update the ensemble of base learners using either a functional gradient descent step, a functional Newton step, or a combination of the two to obtain $F_m(\cdot)= F_{m-1}(\cdot)+ \nu f_m(\cdot)$. We thus avoid the above-mentioned overfitting problem while being computationally more effective compared to doing coordinate descent in both directions. Note that $\Psi_m$ in Algorithm \ref{gpboost_algo} denotes the covariance matrix $\Psi$ for the covariance parameters $\theta_m$ of iteration $m$.
\begin{algorithm}[ht!]
	\SetKwInOut{Input}{Input}
	\SetKwInOut{Output}{Output}
	\Input{Initial value $\theta_0\in\Theta$, learning rate $\nu>0$, number of boosting iterations $M\in\mathbb{N}$, $\texttt{BoostType}\in \{\texttt{"gradient", "newton", "hybrid"}\} $, $\texttt{NesterovAccel}\in\{\texttt{True, False}\}$, and if $\texttt{NesterovAccel==True}$ momentum sequence $\mu_m\in(0,1]$}
	\Output{Predictor function $\hat F(\cdot) = F_{M}(\cdot)$ and covariance parameters $\hat \theta = \theta_M$}
	\caption{GPBoost: Gaussian Process Boosting}\label{gpboost_algo}
	\begin{algorithmic}[1]
		\STATE Initialize $F_0(\cdot)=\argmin_{c\in\mathbb{R}}L(y,c\cdot 1,\theta_0)$
		\FOR{$m=1$ {\bfseries to} $M$}
		\STATE Find $\theta_m=\argmin_{\theta\in\Theta}L(y,F_{m-1},\theta)$ using a method for convex optimization initialized with $\theta_{m-1}$
		\IF{$\texttt{NesterovAccel==True}$}
		\STATE Set $G_{m-1}(\cdot)=F_{m-1}(\cdot)$
		\IF{$m>1$}
		\STATE Update $F_{m-1}(\cdot)=G_{m-1}(\cdot)+\mu_m(G_{m-1}(\cdot)-G_{m-2}(\cdot))$
		\ENDIF
		\ENDIF
		\IF{$\texttt{BoostType=="gradient"}$}
		\STATE Find $f_m(\cdot)=\argmin_{f(\cdot)\in \mathcal{S}}\left\|\Psi_m^{-1}(F_{m-1}-y)-f\right\|^2$
		\ELSIF{$\texttt{BoostType=="newton"}$}
		\STATE Find $f_m(\cdot)=\argmin_{f(\cdot)\in \mathcal{S}}\left(y-F_{m-1}-f\right)^T\Psi_m^{-1}\left(y-F_{m-1}-f\right)$
		\ELSIF{$\texttt{BoostType=="hybrid"}$}
		\STATE Find $\alpha_m=\hat\alpha$, $(\hat\alpha,\hat\gamma)=\argmin_{(\alpha,\gamma):f(\cdot)=h(\cdot;\alpha)^T\gamma \in \mathcal{S}}\left\|\Psi_m^{-1}(F_{m-1}-y)-f\right\|^2$
		\STATE Calculate $\gamma_m =\left(h_{\alpha_m}^T \Psi_m^{-1} h_{\alpha_m}\right)^{-1} h_{\alpha_m}^T\Psi_m^{-1}(y-F_{m-1})$
		\STATE Set $f_m(\cdot)=h(\cdot;\alpha_m)^T\gamma_m$
		\ENDIF
		\STATE Update $F_m(\cdot)= F_{m-1}(\cdot)+ \nu f_m(\cdot)$
		\ENDFOR
	\end{algorithmic}
\end{algorithm}

The coordinate descent step for finding $\theta_m=\argmin_{\theta\in\Theta}L(y,F_{m-1},\theta)$ can be done using a first- or second-order method for convex optimization initialized with the covariance parameters $\theta_{m-1}$ of the previous iteration. In doing so, we avoid the full re-estimation of the covariance parameters in every boosting iteration. Examples of optimization algorithms for determining $\theta_m$ include various forms of gradient descent and quasi-Newton methods such as Fisher scoring which is often called ``natural gradient descent" in machine learning \citep{amari1998natural}. 

For first- and second-order optimization methods, the gradient of $L(y,F,\theta)$ with respect to $\theta$ is required. This gradient is given by
\begin{equation}\label{grad_covpar}
\frac{\partial L(y,F,\theta)}{\partial \theta_k}=
-\frac{1}{2}(y-F)^T{\Psi}^{-1}\frac{\partial\Psi}{\partial \theta_k}{\Psi}^{-1}(y-F)+ \frac{1}{2}\text{tr}\left({\Psi}^{-1}\frac{\partial\Psi}{\partial \theta_k}\right),~~k=1,\dots,q.
\end{equation}
In our software implementation (see Section \ref{software}), we reparametrize all parameters $\theta_k$ with positivity constraints, such as marginal variance and range parameters, on the log-scale $\log(\theta_k)$ in order to constrain them to positive values during the numerical optimization. Further, there is an explicit solution for the error variance parameter: 
\begin{equation}\label{exact_sigma}
\sigma^2=\frac{1}{n}\left(y-F_{m-1}\right)^T{\left(\Psi^\dagger\right)}^{-1}\left(y-F_{m-1}\right).
\end{equation}
For Gaussian processes, we have found in simulated experiments that when doing gradient descent, profiling out the error variance using the analytic formula in \eqref{exact_sigma} increases convergence speed (results not tabulated). For Fisher scoring, on the other hand, profiling out the error variance can reduce convergence speed considerably (results not tabulated).

The Fisher information matrix $I\in \mathbb{R}^{q\times q}$ for Fisher scoring is given by
$$(I)_{kl} = \frac{1}{2}\text{tr}\left({\Psi}^{-1}\frac{\partial\Psi}{\partial \theta_k}{\Psi}^{-1}\frac{\partial\Psi}{\partial \theta_l}\right),~~~~1\leq k,l\leq q.$$
For the linear case, $F(X)=X^T\beta$, asymptotic theory \citep{stein1999} suggests that if the smallest eigenvalue of the Fisher information $I$ tends to infinity as $n\rightarrow \infty$, we can expect that
$$I(\hat\theta)^{1/2}(\hat \theta-\theta_0) \overset{d}{\rightarrow} \mathcal{N}(0, I),$$
where $\theta_0$ denotes the population parameter and $I(\hat\theta)^{1/2}$ is a matrix square root. Based on this, one can construct approximate confidence sets or intervals for $\theta$.

If the risk functional $R(F(\cdot),\theta)=L(y,F,\theta)\Big|_{F=F(X)}$ is convex in its two arguments $F(\cdot)$ and $\theta$ and $\Theta$ is a convex set, then \eqref{optim_def} is a convex optimization problem since $\mathcal{H} = span(\mathcal{S})$ is also convex. Such an algorithm with gradient or Newton steps in $F(\cdot)$ and coordinate descent in $\theta$ thus converges to a minimizer of $R(F(\cdot),\theta)$ as long as the learning rate $\nu$ is not too large to avoid overshooting, i.e., that the risk increases when doing too large steps. 

The computational complexity of the GPBoost algorithm depends on the specific random effects model. For Gaussian processes, computational costs are dominated by the learning of the covariance parameters and the calculation of derivatives with respect to $F$. Specifically, the computational time and space complexity are $O(m^3)$ and $O(m^2)$, respectively, of a single evaluation of the likelihood or its gradient using the Cholesky decomposition when not applying an approximation for large data. Compared to this, the learning of trees is faster in the gradient boosting version and, except for the leaf updates, also in the hybrid gradient-Newton version.

In linear mixed effects models, $L(y,F,\theta)$ is usually optimized by first profiling out the fixed effect part and then optimizing over $\theta$. In our case, this is not an option since there is no explicit solution for $F(\cdot)$ conditional on $\theta$. 
Also, note that restricted maximum likelihood (REML) estimation is often used for linear mixed effects models since otherwise covariance parameter estimates can be biased. This is, however, not applicable to our approach. 

\subsection{Out-of-sample learning for covariance parameters}\label{gpboostoost}
It has recently been observed for both regression and classification that state-of-the-art machine learning techniques such as neural networks, kernel machines, or boosting algorithms can achieve zero training loss and interpolate the training data while at the same time having excellent generalization properties \citep{zhang2016understanding, wyner2017explaining, belkin18a, belkin2019reconciling, bartlett2020benign}. In line with this, we find in our simulated experiments in Section \ref{simul} that estimates of the error variance $\sigma^2$ are often too small also when doing early stopping by monitoring a validation error. 

A way to alleviate this potential variance parameter bias problem is to estimate the covariance parameters using out-of-sample validation data obtained by applying cross-validation or by partitioning the data into two disjoint training and validation sets. To avoid that the predictor function and the covariance parameters $\theta$ are only learned on a fraction of the full data, we propose a two-step approach presented in the GPBoostOOS Algorithm \ref{gpboostoos_algo}. In brief, the GPBoostOOS algorithm first runs the GPBoost algorithm on subsamples of the data and obtains predictions for the predictor function on the left-out validation data. The covariance parameters are then estimated on the out-of-sample data using the predicted predictor function. Finally, the GPBoost algorithm is run a second time on the full data without estimating the covariance parameters. When $k$-fold cross-validation is used in steps 1.-3. of Algorithm \ref{gpboostoos_algo}, both the predictor function and the covariance parameters are learned using the full data.
\begin{algorithm}[ht!]
	\SetKwInOut{Input}{Input}
	\SetKwInOut{Output}{Output}
	\Input{Initial value $\theta_0\in\Theta$, learning rate $\nu>0$, number of boosting iterations $M\in\mathbb{N}$, $\texttt{BoostType}\in \{\texttt{"gradient", "newton", "hybrid"}\} $, $\texttt{NesterovAccel}\in\{\texttt{True, False}\}$, and if $\texttt{NesterovAccel==True}$ momentum sequence $\mu_m\in(0,1]$}
	\Output{Predictor function $\hat F(\cdot)$ and covariance parameters $\hat \theta$}
	\caption{GPBoostOOS: Gaussian Process Boosting with Out-Of-Sample covariance parameter estimation}\label{gpboostoos_algo}
	\begin{algorithmic}[1]
		\STATE Partition the data into training and validation sets, e.g., using $k$-fold cross-validation or by splitting the data into two disjoint sets
		\STATE Run the GPBoost algorithm on the training data and generate predictions for the predictor function on the validation data $\hat F_{val}$
		\STATE Find $\hat \theta=\argmin_{\theta\in\Theta}L(y_{val},\hat F_{val},\theta)$ using the validation data with response $y_{val}$
		\STATE Run the GPBoost algorithm on the full data while holding the covariance parameters $\theta$ fixed at $\hat \theta$, i.e., by skipping line 3 in Algorithm \ref{gpboost_algo}, to obtain $\hat F(\cdot)$
	\end{algorithmic}
\end{algorithm}

\subsection{Computationally efficient learning for large data}\label{largedata}
For computationally efficient learning of trees, several approaches exist so that computations scale well to large data \citep{chen2016xgboost, ke2017lightgbm, CatBoost2017}. In this article, we use the approach of \citet{ke2017lightgbm}. 

Concerning covariance parameters, we adopt the following solutions for reducing computational costs. If the random effects $b$ consist of only grouped random effects, $\Psi$ is usually a sparse matrix, and computations can be done efficiently using sparse matrix algebra. Further, we can use the Sherman-Morrison-Woodbury formula
\begin{equation}\label{woodbury}
\left(Z\Sigma Z^T + \sigma^2 I_{n}\right)^{-1}=\sigma^{-2}I_n-\sigma^{-2}Z\left(\sigma^2{\Sigma}^{-1}+Z^T Z\right)^{-1}Z^T
\end{equation}
for calculating gradients with respect to $F$ and $\theta$ and for evaluating the log-likelihood since the dimension of the random effects $m$ is typically smaller than the number of samples $n$.

If $b$ contains a Gaussian process with a non-sparse covariance matrix, both the computational cost and the required memory do not scale well in the number of observed locations $m$ as standard approaches relying on the Cholesky factorization require $O(m^3)$ calculations and $O(m^2)$ memory storage. In this case, one has to use some approximation to make calculations feasible. We choose to use Vecchia's approximation \citep{vecchia1988estimation, datta2016hierarchical, katzfuss2017general, finley2019efficient}, also denoted as nearest-neighbor Gaussian process (NNGP) model \citep{datta2016hierarchical}, as it is very accurate for spatial data, embarrassingly parallel, and it has the desirable property that maximizing it corresponds to solving a set of unbiased estimating equations \citep{guinness2018permutation}. This has led some authors in the spatial statistics community to declare that ``[a]mong the sea of Gaussian process approximations proposed over the past several decades, Vecchia's approximation has emerged as a leader" \citep{guinness2019gaussian}. However, we note that this is not the only large-data Gaussian process approximation that can be applied to the GPBoost algorithm. Other potential Gaussian process approximations include \citet{snelson2006sparse, quinonero2007approximation, cunningham2008fast, titsias2009variational, hensman2013gaussian, wilson2015kernel, gardner2018gpytorch}; see also the review of \cite{liu2020gaussian}.

Intuitively, the idea of Vecchia's approximation is to approximate a Cholesky factor of the precision matrix using a sparse matrix and thus to obtain a sparse approximate precision matrix. In the following, we briefly review how this is obtained in our case and then show how gradients of the negative log-likelihood given in  \eqref{grad_covpar} can be calculated efficiently. To the best of our knowledge, the latter result is novel.


\subsubsection{Vecchia approximation for the response variable $y$}\label{vecchia_resp}
Vecchia approximations can be seen as a special form of composite likelihood methods \citep{varin2011overview}. In our case, the likelihood $p(y|F,\theta)$ is approximated as
\begin{equation}\label{vecchia}
\begin{split}
p(y|F,\theta)&=\prod_{i=1}^n p(y_{i}|(y_1,\dots,y_{i-1}),F,\theta)\\
&\approx \prod_{i=1}^n p(y_i|y_{N(i)},F,\theta),
\end{split}
\end{equation}
where $y_{N(i)}$ are subsets of the conditioning sets $(y_1,\dots,y_{i-1})$, and $N(i)$ denotes the corresponding subsets of indices. As is commonly done, we choose $N(i)$ as the indices of the $m_v$ nearest neighbors of $s_i$ among $s_1,\dots,s_{i-1}$ if $i>m_v+1$, and, in the case $i\leq m_v+1$, $N(i)$ equals $(1,\dots,i-1)$.

By standard arguments for conditional Gaussian distributions, we have
\begin{equation*}
p(y_i|y_{N(i)},F,\theta)=\mathcal{N}\left(y_i\mid F_i + A_i\left(y_{N(i)}-F_{N(i)}\right),D_i\right),
\end{equation*}
where $\mathcal{N}(x|\mu,\Xi)$ denotes a Gaussian density with mean vector $\mu$ and covariance matrix $\Xi$ evaluated at $x$, and the matrices $A_i\in\mathbb{R}^{1\times|N(i)|}$ and $D_i\in\mathbb{R}$, where $|N(i)|$ denotes the size of the set $N(i)$, are given by
\begin{equation}\label{def_B_D}
\begin{split}
A_i&=\left(Z\Sigma Z^T\right)_{i,N(i)}\left(\left(Z\Sigma Z^T+\sigma^2I_n\right)_{N(i)}\right)^{-1},\\
D_i&=\left(Z\Sigma Z^T+\sigma^2I_n\right)_{i,i}-A_i\left(Z\Sigma Z^T\right)_{N(i),i},
\end{split}
\end{equation}
where $\Sigma=(c(s_l,s_k))_{l,k}$, $1\leq l,k\leq n$ is the covariance matrix of $b$, $c(\cdot,\cdot)$ is the covariance function, $M_{i,N(i)}$ denotes the sub-matrix of a matrix $M$ consisting of row $i$ and columns $N(i)$, and $M_{N(i)}$ denotes the sub-matrix of a matrix $M$ consisting of rows $N(i)$ and columns $N(i)$. Note that if $Z$ is a diagonal matrix, we have
$$\left(Z\Sigma Z^T\right)_{N(i)}=Z_{N(i)}\Sigma_{N(i)}Z_{N(i)}=\Sigma_{N(i)}\odot\left(zz^T\right)_{N(i)},$$
where $z$ is the diagonal of $Z$ and $\odot$ denotes the Hadamard product \citep{dambon2020maximum}. Using this relationship can lead to a reduction in computational cost, in particular for random coefficient models.

We further denote by $B$ the lower triangular matrix with $1$'s on the diagonal, off-diagonal entries 
\begin{equation}\label{def_B}
(B)_{i,N(i)}=-A_i,
\end{equation}
and $0$'s otherwise, and by $D$ a diagonal matrix with $D_i$ on the diagonal. We then obtain the following approximate distribution
\begin{equation}\label{vecchia_approx}
y\overset{approx}{\sim} \mathcal{N}\left(F(X),\tilde {\Psi}\right),~~~~\tilde {\Psi}=B^{-1}DB^{-T},
\end{equation}
and the corresponding precision matrix is given by
\begin{equation}\label{vecchia_approx_prec}
\tilde {\Psi}^{-1}=B^T{D}^{-1}B,
\end{equation}
where $B$ and $\tilde {\Psi}^{-1}$ are sparse. 

Concerning computational complexity, the main burden is the calculation of Cholesky factors of $\text{Cov}(y_{N(i)})=\left(Z\Sigma Z^T+\sigma^2I_n\right)_{N(i)}$. Calculating a Vecchia approximation has $O(nm_v^3)$ computational cost and requires $O(nm_v)$ memory storage. Thus, for given $m_v$, both the computational time and the memory storage grow linearly. Concerning the choice of the numbers of neighbors $m_v$, \citet{datta2016hierarchical} report that  ``usually a small value of [$m_v$] between $10$ and $1$5 produces performance at par with a full geostatistical model''.

\subsubsection{Efficient calculation of the gradient and Fisher information for the Vecchia approximation}
In the following, we show how the gradient and the Fisher information of the approximate log-likelihood of the Vecchia approximation given in \eqref{vecchia_approx} can be calculated efficiently. To the best of our knowledge, the following results are novel. \citet{guinness2019gaussian} also presents a way for computing the gradient and Fisher information for the Vecchia approximation. However, \citet{guinness2019gaussian} uses a different representation of the approximate likelihood by writing conditional densities in \eqref{vecchia} as ratios of joint and marginal densities and, in doing so, obtains a different way for calculating the gradient and Fisher information compared to our result. \citet{guinness2019gaussian} motivates his approach by claiming that for calculating the gradient in \eqref{grad_covpar}, "[n]ot only is $\left[\frac{\partial\Psi}{\partial \theta_k}\right]$ too large to store in memory, the covariances $\left[\Psi\right]$ are not easily computable, nor are their partial derivatives". The following Proposition \ref{GradVecchia} and its proof show that $\frac{\partial\Psi}{\partial \theta_k}$ does not need to be stored in memory, and neither $\Psi$ nor its partial derivatives need to be computed. In the approach of \citet{guinness2019gaussian} for calculating the gradient, the computational complexity is dominated by the need to calculate two Cholesky factorizations of matrices of sizes $\mathbb{R}^{|N(i)|\times |N(i)|}$ and $\mathbb{R}^{(|N(i)|+1)\times (|N(i)|+1)}$ for every data point $i$, where we recall that $|N(i)|$ denotes the number of neighbors of sample $i$. In contrast, in our approach in Proposition \ref{GradVecchia} below, only one matrix of size $\mathbb{R}^{|N(i)|\times |N(i)|}$ needs to be factorized for every sample $i$. This means that our approach for calculating the gradient has approximately only half the computational cost compared to the one of \citet{guinness2019gaussian}.

\begin{proposition}\label{GradVecchia}
	The gradient of the negative log-likelihood $\tilde L(y,F,\theta)$ for the Vecchia approximation given in \eqref{vecchia_approx} can be calculated as
	$$\frac{\partial \tilde L(y,F,\theta)}{\partial \theta_k}=\frac{1}{2\sigma^2} \left(2u_k^Tu - u^T\frac{\partial D}{\partial \theta_k} u\right) +\frac{1}{2}  \sum_{i=1}^n\frac{1}{D_i}\frac{\partial D_i}{\partial \theta_k},~~1\leq k\leq q,$$
	where
	\begin{equation}\label{zzk_V}
	u={D}^{-1}B(y-F) ~~~~\text{and}~~~~ u_k=\frac{\partial B}{\partial \theta_k}(y-F),
	\end{equation}
	and $\frac{\partial B}{\partial \theta_k}$ are lower triangular and $\frac{\partial D}{\partial \theta_k}$ diagonal matrices with non-zero entries given by
	\begin{equation*}
	\begin{split}
	\left(\frac{\partial B}{\partial \theta_k}\right)_{i,N(i)} = & -\frac{\partial A_i}{\partial \theta_k}\\
	= &- \left(Z\frac{\partial \Sigma}{\partial \theta_k} Z^T\right)_{i,N(i)}\left(\left(Z \Sigma Z^T + \sigma^2I_n\right)_{N(i)}\right)^{-1}\\
	&+\left(Z \Sigma Z^T\right)_{i,N(i)}\left(\left(Z \Sigma Z^T + \sigma^2I_n\right)_{N(i)}\right)^{-1} \left(Z \frac{\partial \Sigma}{\partial \theta_k} Z^T\right)_{N(i)} \left(\left(Z \Sigma Z^T + \sigma^2I_n\right)_{N(i)}\right)^{-1},\\
	\frac{\partial D_i}{\partial \theta_k} = & \left(Z\frac{\partial \Sigma}{\partial \theta_k} Z^T\right)_{i,i} - \frac{\partial A_i}{\partial \theta_k}\left(Z^T \Sigma Z\right)_{N(i),i} -A_i\left(Z^T \frac{\partial \Sigma}{\partial \theta_k} Z\right)_{N(i),i},
	\end{split}
	\end{equation*}
	for $1< k\leq q$, and for $k=1$, the non-zero entries of $\frac{\partial B}{\partial \theta_k}$ and $\frac{\partial D}{\partial \theta_k}$ are
	\begin{equation*}
	\begin{split}
	\left(\frac{\partial B}{\partial \sigma^2}\right)_{i,N(i)} =& \left(Z \Sigma Z^T\right)_{i,N(i)}\left(\left(Z \Sigma Z^T + \sigma^2I_n\right)_{N(i)}\right)^{-2},\\
	\frac{\partial D_i}{\partial \sigma^2} =& 1 - \frac{\partial A_i}{\partial \sigma^2}\left(Z^T \Sigma Z\right)_{N(i),i}.
	\end{split}
	\end{equation*}
	
\end{proposition}
A proof can be found in Appendix \ref{proofs}. As indicated above, the computational costs for calculating the gradient are $O(nm_v^3)$. The Fisher information for the Vecchia approximation in \eqref{vecchia_approx} can be calculated using the following result. 
\begin{proposition}\label{FIVecchia}
	The Fisher information for the Vecchia approximation matrix in \eqref{vecchia_approx} has entries
	\begin{equation}\label{FIVecciaentry}
	(I)_{kl}=\sum_{i,j=1}^n\left({D}^{-1} \frac{\partial B}{\partial \theta_k} B^{-1} \right)_{ij}\left(\frac{\partial B}{\partial \theta_l}B^{-1}D \right)_{ij} + \frac{1}{2}\sum_{i=1}^n {D_i}^{-2}\frac{\partial D_i}{\partial \theta_k}\frac{\partial D_i}{\partial \theta_l},~~1\leq k \leq q,
	\end{equation}
	where $\frac{\partial B}{\partial \theta_k}$ and $\frac{\partial D}{\partial \theta_k}$ are lower triangular and diagonal matrices defined in Proposition \ref{GradVecchia}

\end{proposition}
A proof can be found in Appendix \ref{proofs}. This Fisher information can be used for finding a maximum of the (approximate) likelihood using Fisher scoring. Concerning an approximate covariance matrix for $\hat \theta$ using the asymptotic result mentioned in Section \ref{GPB}, we note that since the Vecchia approximation results in a misspecified model, the Fisher information matrix needs to be replaced by the Godambe information matrix \citep{godambe1960optimum} $G=H I^{-1}H$, where $H$ is the negative expected Hessian of the log-likelihood.

\subsection{Prediction}\label{preds}
Let $y_p\in\mathbb{R}^{n_p}$ denote the random variables for which predictions should be made. We have 
\begin{equation}\label{pred_obs_dis}
\begin{split}
\begin{pmatrix} y \\ y_p\end{pmatrix} &= \begin{pmatrix} F(X) \\ F(X_p)\end{pmatrix} + 
\begin{pmatrix} (Z,0_{n\times m_p}) \\ Z_p \end{pmatrix}\begin{pmatrix} b \\ b_p \end{pmatrix}+\begin{pmatrix} \epsilon \\ \epsilon_p \end{pmatrix},\\
&\sim \mathcal{N} \left(\begin{pmatrix} F(X) \\ F(X_p)\end{pmatrix},
 \begin{pmatrix} Z\Sigma Z^T + \sigma^2 I_{n}& Z(\Sigma,\Sigma_{op})Z_p^T \\ Z_p(\Sigma,\Sigma_{op})^TZ^T &Z_p\begin{pmatrix} \Sigma& \Sigma_{op}\\ \Sigma_{op}^T&\Sigma_p\end{pmatrix}Z_p^T+ \sigma^2 I_{n_p}\end{pmatrix}\right)
\end{split}
\end{equation}
where $b_p\in\mathbb{R}^{m_p}$ is a vector of $m_p$ random effects, for which no data has been observed in $y$, $(Z,0_{n\times m_p})\in \mathbb{R}^{n\times (m + m_p)}$, $0_{n\times m_p}\in\mathbb{R}^{n\times m_p}$ is a matrix of zeros, the matrix $Z_p\in \mathbb{R}^{n_p\times (m + m_p)}$ relates the vector of observed and new random effects $(b^T,b_p^T)^T\in\mathbb{R}^{m + m_p}$ to $y_p$, $(\Sigma,\Sigma_{op})\in\mathbb{R}^{m\times (m + m_p)}$, $\Sigma_{op} = \text{Cov}(b,b_p)$, $\Sigma_{p} = \text{Cov}(b_p)$, and $X_p\in\mathbb{R}^{n_p\times p}$ is the predictor variable matrix of the predictions. 

Note that $y_p$ can be related to both existing random effects $b$, for which data $y$ has been observed, and also new, unobserved random effects $b_p$. This distinction is particularly relevant for grouped random effects, where predictions can be made for samples of groups for which data has already been observed in $y$, or also for new groups for which no data has been observed in $y$. Further, this can also be useful for Gaussian process models for distinguishing when predictions should be made for $y_p$, or $b_p$, at new locations, or when $b$ should be predicted at observed locations. 

From \eqref{pred_obs_dis}, it follows that the conditional distribution $y_p|y$ is given by 
\[y_p|y\sim \mathcal{N} \left( \mu_p , \Xi_p\right),\]
where
\begin{equation}\label{predeq}
\begin{split}
\mu_p=& F(X_p)+  Z_p(\Sigma,\Sigma_{op})^TZ^T \left(Z\Sigma Z^T + \sigma^2I_{n}\right)^{-1}\left(y-F(X)\right)\\
\Xi_p=&  Z_p\begin{pmatrix} \Sigma& \Sigma_{op}\\ \Sigma_{op}^T&\Sigma_p\end{pmatrix}Z_p^T+ \sigma^2 I_{n_p} - Z_p(\Sigma,\Sigma_{op})^TZ^T \left(Z\Sigma Z^T + \sigma^2I_{n}\right)^{-1}  Z(\Sigma,\Sigma_{op})Z_p^T.
\end{split}
\end{equation}

Depending on the application, if $n\gg m$, the above quantities can be more efficiently calculated using the Sherman-Morrison-Woodbury formula given in \eqref{woodbury}. Further, predictions for the latent $b$, $b_p$, or $F(X_p)+Z_p\begin{pmatrix} b \\ b_p \end{pmatrix}$ can be done analogously with minor modifications, e.g., dropping the error variance term $\sigma^2I_{n_p}$ from the covariance matrix in \eqref{predeq}.

\subsubsection{Prediction using the Vecchia approximation}\label{largedata2}
Similarly as for parameter estimation, Vecchia approximations can also be used for making predictions. Specifically, predictions can be obtained by applying a Vecchia approximation to the joint response vector of observed and prediction locations. When doing so, one has to choose an ordering among the joint set of observed and predicted locations. We assume that either the observed or the prediction locations appear first in the ordering of the response variable. The former has the advantage that the nearest neighbors found for estimation can be reused and that the predictive distributions have the simple form given below in \eqref{predVechOF}. On the other hand, if prediction locations appear first in the ordering, the approximations of predictive distributions are generally more accurate. See \citet{katzfuss2018vecchia} for a comparison of different approaches for making predictions with Vecchia approximations.

\begin{proposition}\label{PredVEcchiaOF}
	Assume that prediction are made at $n_p$ locations $s_{p,1},\dots,s_{p,n_p}$ with predictor variable data $X_p$. When applying the Vecchia approximation in \eqref{vecchia_approx} to the response vector $(y,y_p)^T$ with the observed response $y$ appearing first in the ordering, the conditional distribution $y_p|y$ is given by 
	\[y_p|y\sim \mathcal{N} \left( \mu_p , \Xi_p\right),\]
	where
	\begin{equation}\label{predVechOF}
	\begin{split}
	\mu_p=& F(X_p)-B_p^{-1}B_{po}\left(y-F(X)\right),\\
	\Xi_p=& B_p^{-1}{D_p}B_p^{-T},
	\end{split}
	\end{equation}
	and $B_{po}\in\mathbb{R}^{n_p\times n}$, $B_p\in\mathbb{R}^{n_p\times n_p}$, ${D_p}^{-1}\in\mathbb{R}^{n_p\times n_p}$ are the following submatrices of the Vecchia approximated precision matrix $\tilde{Cov}\left((y,y_p)^T\right)^{-1}$:
	\begin{equation}\label{prec_pred}
	\tilde{Cov}\left((y,y_p)^T\right)^{-1}
	=\begin{pmatrix} B & 0 \\ B_{po}&B_p\end{pmatrix}^T
	\begin{pmatrix} {D}^{-1}& 0 \\ 0&{D_p}^{-1}\end{pmatrix}
	\begin{pmatrix} B & 0 \\ B_{po}&B_p\end{pmatrix},
	\end{equation}
	and $B$ and $D$ are defined in \eqref{def_B_D} and \eqref{def_B}.
\end{proposition}

A proof can be found in Appendix \ref{proofs}. Note that $D_p$ is a diagonal matrix, $B_p$ is a lower triangular matrix with $1$'s on the diagonal and non-zero off-diagonal entries corresponding to the nearest neighbors of the prediction locations among the prediction locations themselves $s_{p,1},\dots,s_{p,n_p}$, and $B_{po}$ has non-zero entries corresponding to the nearest neighbors of the prediction locations among the observed locations $s_1,\dots,s_{n}$. 

If only univariate predictive distributions are of interest, computational costs can be additionally reduced by restricting that one conditions on observed locations only in \eqref{vecchia}. The latter means that for every prediction location $s_{p,i}$, one conditions only on observed data $y_{N(i)}$ where $N(i)$ denotes the set of nearest neighbors for location $s_{p,i}$. In this case, $B_p$ is an identity matrix and the predictive covariance matrix $\Xi_p$ is a diagonal matrix. The latter can be a drawback if multivariate predictive distributions are required.

When prediction locations appear first in the ordering of the response variable, predictions can be obtained as follows.
\begin{proposition}\label{PredVEcchiaPF}
	Assume that prediction are made at $n_p$ locations $s_{p,1},\dots,s_{p,n_p}$ with predictor variable data $X_p$. When applying the Vecchia approximation in \eqref{vecchia_approx} to the response vector $(y_p,y)^T$ with the predicted response $y_p$ appearing first in the ordering, the conditional distribution $y_p|y$ is given by 
	\[y_p|y\sim \mathcal{N} \left( \mu_p , \Xi_p\right),\]
	with
	\begin{equation}\label{predVechLF}
	\begin{split}
	\mu_p=& F(X_p)- \left(B_p^T{D_p}^{-1}B_p + B_{op}^T{D_o}^{-1}B_{op}\right)^{-1}B_{op}^T{D_o}^{-1}B_o\left(y-F(X)\right),\\
	\Xi_p=& \left(B_p^T{D_p}^{-1}B_p + B_{op}^T{D_o}^{-1}B_{op}\right)^{-1},
	\end{split}
	\end{equation}
	where $B_{o},D_o\in\mathbb{R}^{n\times n}$, $B_{op}\in\mathbb{R}^{n\times n_p}$, $B_p,D_p\in\mathbb{R}^{n_p\times n_p}$, ${D_p}^{-1}\in\mathbb{R}^{n_p\times n_p}$ are the following submatrices of the Vecchia approximated precision matrix $\tilde{Cov}\left((y_p,y)^T\right)^{-1}$:
	$$
	\tilde{Cov}\left((y_p,y)^T\right)^{-1}=    \begin{pmatrix} B_p & 0 \\ B_{op}&B_o\end{pmatrix}^T
	\begin{pmatrix} {D_p}^{-1} & 0 \\0&{D_o}^{-1}\end{pmatrix} 
	\begin{pmatrix} B_p & 0 \\ B_{op}&B_o\end{pmatrix}.
	$$
\end{proposition}
A proof can be found in Appendix \ref{proofs}. 

\subsection{Software implementation}\label{software}
The GPBoost algorithm is implemented in the \texttt{GPBoost} library written in C++ with a C application programming interface (API) and corresponding Python and R packages. See \url{https://github.com/fabsig/GPBoost} for more information. For linear algebra calculations, we rely on the \texttt{Eigen} library \citep{eigenweb}. Sparse matrix algebra is used, in particular for calculating Cholesky decompositions, whenever covariance matrices are sparse, e.g., in the case of grouped random effects. In addition, to speed up computations for solving sparse linear triangular equation systems where the right-hand side is also sparse, we use the function \texttt{cs\_spsolve} from the \texttt{CSparse} library \citep{davis2005csparse} where the non-zero entries of the solutions are determined using a depth-first search algorithm. Further, multi-processor parallelization is done using \texttt{OpenMP}. For the tree-boosting part, in particular the tree growing algorithm, we use the \texttt{LightGBM} library \citep{ke2017lightgbm}.  Note that the \texttt{GPBoost} library allows for modeling Gaussian processes, grouped random effects including hierarchically nested and crossed ones, random coefficients, and combinations of the former.

\section{Simulated Experiments}\label{simul}
In the following, we use simulation to investigate the prediction accuracy as well as the properties of the covariances parameter and predictor function estimates of the GPBoost algorithm.

\subsection{Simulation setting}\label{sim_set}
We simulate data from the model given in \eqref{modunit}. For the random effects part $Zb$, we use both a grouped random effects model and a spatial Gaussian process model. The sample size is $n=5000$ for the grouped data and $n=500$ for the spatial data. The reason for using a smaller sample size for the spatial data is that this allows us to do all calculations exactly without relying on an approximation for large data. In Section \ref{estimation}, we show how learning can be done for large data, and we use this in the application in Section \ref{data_appl}.

For simulating grouped data, we use a single-level grouped random effects model with $10$ samples per group, i.e., $m=500$ different groups. In other words, there is a single high-cardinality categorical variable with $500$ different categories. For simulating spatial data, we use a spatial Gaussian process model with an exponential covariance function 
\begin{equation*}\label{exp_cov}
c(s,s')=\sigma_1^2 \exp(-\|s-s'\|/\rho),
\end{equation*}
where the locations $s$ are in $[0,1]^2$ and $\rho=0.1$. The marginal variance in both models is set to $\sigma^2_1=1$, and the error variance equals $\sigma^2=1$ such that the signal-to-noise ratio between the random effects $Zb$ and the error term $\epsilon$ is $1$. 

Concerning the fixed effects predictor function $F(\cdot)$ and the predictor variables $X$, we use the following different specifications:
\begin{alignat*}{3}
F(x)&=C\cdot(2x_1+x_2^2+4\cdot \mathbbm{1}_{\{x_3>0\}}+2\log(|x_1|)x_3),~~x=(x_1,\dots,x_9)^T,~~x\sim \mathcal{N} ( 0 , I_9), &(\text{`hajjem'})\\
F(x)&=C\cdot \tan^{-1}\left(\frac{x_2x_3-1-\frac{1}{x_2x_4}}{x_1}\right), ~~x=(x_1,x_2,x_3,x_4)^T, &(\text{`friedman3'}) \\ 
& ~~~~~~~~ x_1\sim Unif(0,100), x_2\sim Unif(40\pi,560\pi), x_3\sim Unif(0,1), x_4\sim Unif(1,11),\\
F(x)&=C\cdot(1+x_1+x_2), ~~x=(x_1,x_2)^T,~~ x_1,x_2\overset{iid}{\sim} Unif(0,1). &(\text{`linear'})
\end{alignat*}
The function `hajjem' has been used in \citet{hajjem2014mixed} to compare non-parametric mixed effects models, and the function `friedman3' was first used in \citet{friedman1991multivariate} and has since then often been used to compare non-parametric regression models. We also include a linear function to investigate how our approach compares to a linear mixed effects model when the data generating process is linear. The constant $C$ is chosen such that the variance of $F(x)$ equals approximately $1$, i.e., that $F(x)$ has the same signal strength as the random effects part. Predictor variable data is simulated independently from the random effects.

We simulate $100$ times both a training data set of size $n$ and two different test data sets each also of size $n$. Learning and selection of tuning parameters are done on the training data, and evaluation is done on the test data for all models considered. The two test data sets, denoted briefly as ``interpolation" and ``extrapolation" test sets, are generated as follows in every simulation run. For the grouped data, the ``interpolation" test data set consists of random effects for the same groups as in the training data, and the ``extrapolation" test data contains $m$ independent random effects for new groups that have not been observed in the training data. For the spatial data, the training data locations are sampled uniformly from $[0,1]^2$ excluding $[0.5,1]^2$, the ``interpolation" test data set is obtained by also simulating locations uniformly in the same area, and the ``extrapolation" test data contains locations sampled uniformly from $[0.5,1]^2$. Predictions for the ``extrapolation" test data are thus to some degree extrapolations. Figure \ref{Train_test_locs} illustrates this. 
\begin{figure*}[ht!]
	\centering
	\includegraphics[width=0.65\textwidth]{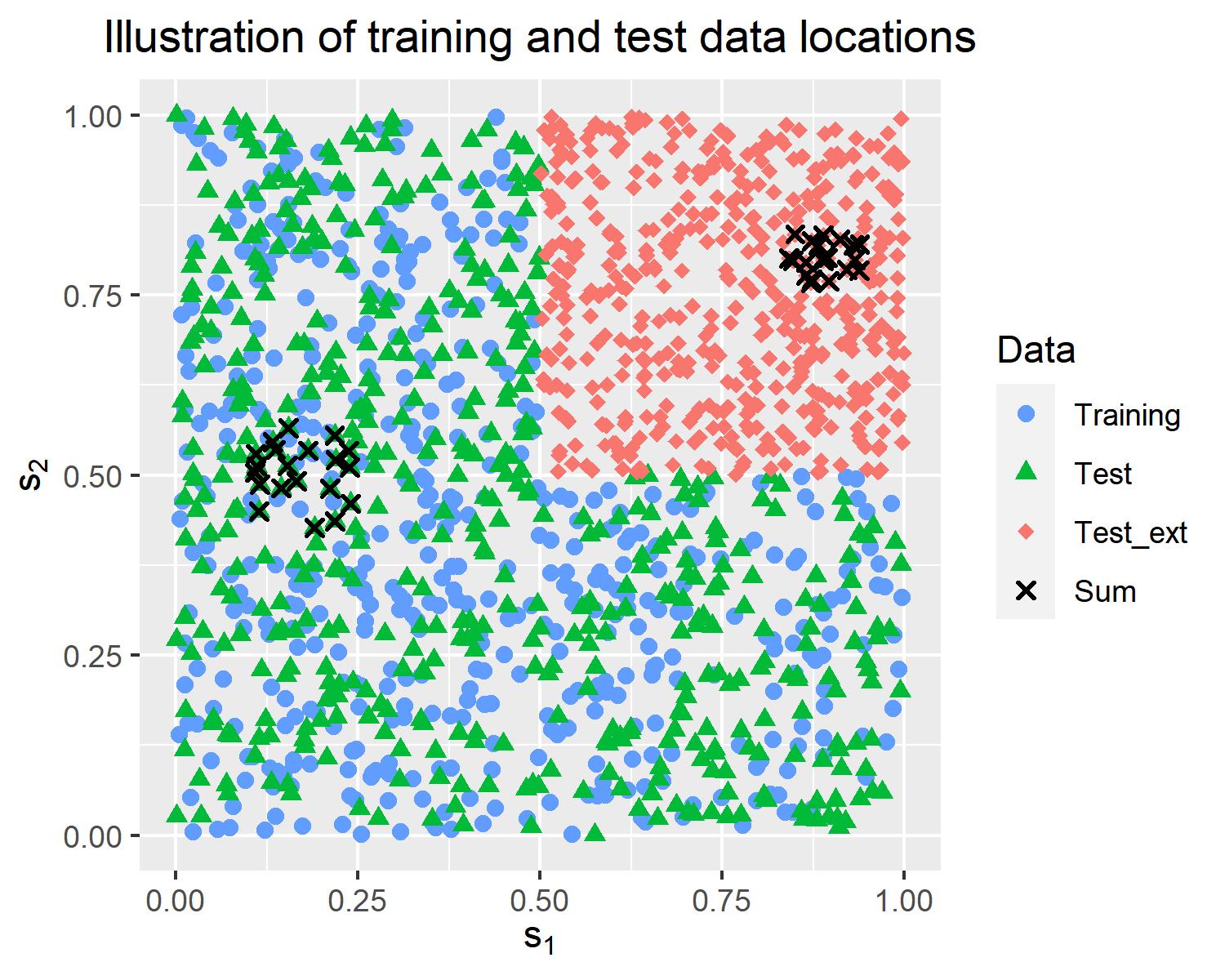}
	\caption{Example of locations for training and test data for the spatial data. ``Test" and ``Test\_ext" refers to locations of the ``interpolation" and ``extrapolation" test data sets, respectively. The black crosses show examples of locations for which predictions of sums are made.}
	\label{Train_test_locs} 
\end{figure*}

\subsection{Methods considered}\label{methods}
We compare the GPBoost algorithm to the following alternative approaches: linear grouped mixed effects models (`LinearME') and linear Gaussian process models (`LinearGP') with $F(X)=X^T\beta$ and an exponential covariance function, mixed-effects random forest (`MERF') \citep{hajjem2014mixed}, RE-EM trees (`REEMtree') \citep{sela2012re}, model-based boosting (`mboost') \citep{hothorn2010model}, independent gradient boosting with a square loss (`LSBoost'), and the boosting approach for categorical predictor variables of \citet{CatBoost2017} (`CatBoost'). Besides the GPBoost algorithm, we apply the GPBoostOOS algorithm where covariance parameters are estimated using $4$-fold cross-validation. For the latter cross-validation, sampling of data indices is done independently of the grouping variable and the spatial coordinates. For the spatial data, we additionally consider a two-step approach (`TwoStep') where $F(\cdot)$ and $\theta$ in \eqref{modunit} are estimated separately in an iterative manner instead of jointly as in the GPBoost algorithm. Specifically, we first apply independent gradient boosting using the predictor variables $X$ and then fit a zero-mean spatial Gaussian process to the residuals of the first step. For all boosting algorithms, we use gradient boosting without Nesterov acceleration and trees as base learners, except for the grouped and spatial random effects in mboost where Ridge regression and splines are used. In independent gradient boosting with a square loss, the spatial locations and the categorical grouping variable are included as additional predictor variables in the predictor function $F(\cdot)$. In doing so, we consider the grouping variable as a categorical variable and use the approach of \texttt{LightGBM} to handle categorical variables. Considering the grouping variable as a continuous variable or using dummy variables leads to worse results (results not tabulated). The MERF, REEMtree, and CatBoost algorithms are only used for the grouped data and not for the spatial data.

Learning and prediction with the GPBoost and GPBoostOOS algorithms, the linear grouped random effects and Gaussian process models, and independent boosting with a square loss (LSBoost) is done using the \texttt{GPBoost} library version 0.7.8 compiled with the MSVC compiler version 19.24.28315.0 and OpenMP version 2.0.\footnote{When using another compiler such as a GNU compiler which supports a higher version of OpenMP, computations for the random effects part can become considerably faster (e.g., approximately twice as fast for the Gaussian process model when using OpenMP 4.5). However, the tree-boosting part of \texttt{LightGBM} is slower when using a GNU compiler.} For the linear mixed effects and Gaussian process models, the GPBoost algorithm, and the GPBoostOOS algorithm, optima for covariance parameters $\theta$ are found using Nesterov accelerated gradient descent \citep{nesterov2004introductory}. For the mboost algorithm, we use the \texttt{mboost} R package \citep{mboost14} version 2.9-2 and model spatial effects using bivariate P-spline base learner (\texttt{bspatial} with \texttt{df=6}) and grouped random effects using random effects base learners (\texttt{brandom} with \texttt{df=4}). All other predictor variables are modeled using trees as base learners. For the MERF algorithm, we use the \texttt{merf} Python package version 0.3\footnote{We have also tried the \texttt{merf} Python package version 1.0. While the results are overall very similar, the estimates of the error variance $\sigma^2$ are highly biased when using version 1.0. For this reason, we use version 0.3.}. The number of iterations of the MERF algorithm is set to $100$. Increasing this value does not change our findings (results not tabulated). However, we note that we often do not observe convergence of the MERF algorithm, no matter how long we let it run.\footnote{This is an observation that has also been made by the creator of the \texttt{MERF} package in a blog post; see https://towardsdatascience.com/mixed-effects-random-forests-6ecbb85cb177 (retrieved on August 22, 2022).} For the REEMtree algorithm, we use the \texttt{REEMtree} R package version 0.90.3. For the CatBoost algorithm, we use the \texttt{CatBoost} library version 1.0.6. All calculations are done on a laptop with a $2.9$ GHz quad-core processor and $16$ GB of random-access memory (RAM).

\subsection{Evaluation criteria}\label{eval_crit}
We measure prediction accuracy using the the root mean square error (RMSE) $$\sqrt{\frac{1}{n_p}\sum_{i=1}^{n_p} (y_{p,i}-\mu_{p,i})^2},$$
where $y_{p,i}$ are test response variables, are $\mu_{p,i}$ predictive means, and $n_p$ is the number of test samples. In addition, we analyze the accuracy of probabilistic predictions for the spatial data. As is commonly done in spatial statistics, probabilistic predictions are evaluated using the continuous ranked probability score (CRPS) \citep{gneiting2007strictly} given by
$$\int \left(F_{p,i}(y) - \mathbbm{1}_{\{y_{p,i} \leq y\}} \right)^2dy $$
for one sample $i$, where $F_{p,i}$ is the predictive cumulative distribution function (CDF). Intuitively, the CRPS corresponds to the mean square error between the predictive CDF $F_{p,i}$ and the empirical CDF $\mathbbm{1}_{\{y_{p,i} \leq y\}}$. The CRPS is a proper scoring rule \citep{gneiting2007strictly} that measures both calibration and sharpness of predictive distributions; see \citet{gneiting2007probabilistic} for more information. In our case, the predictive distributions are Gaussian distributions, and the average CRPS can be calculated explicitly as 
$$\frac{1}{n_p}\sum_{i=1}^{n_p} \sigma_{p,i}\left(\frac{-1}{\sqrt{\pi}}+2\phi\left(\frac{y_{p,i}-\mu_{p,i}}{\sigma_{p,i}}\right)+\frac{y_{p,i}-\mu_{p,i}}{\sigma_{p,i}}\left(2\Phi\left(\frac{y_{p,i}-\mu_{p,i}}{\sigma_{p,i}}\right)-1\right)\right),$$
where $\mu_{p,i}$ and $\sigma_{p,i}^2$ are predictive means and variances, and $\phi$ and $\Phi$ denote the probability density function and CDF of a standard Gaussian variable. Predictive variances $\sigma_{p,i}^2$ are obtained as described in Section \ref{preds} for the Gaussian process-based models. For the deterministic independent boosting approaches (mboost and LSBoost), we use sample variances of the residuals on the training data as predictive variances.   

For the spatial data, we additionally evaluate the accuracy when predicting sums of $n'=20$ observations $y_p=(y_{p,1},\dots,y_{p,n'})^T$. For instance, such predictions are required in meteorology for predicting the total precipitation over a catchment area or in real estate for predicting the total value of a portfolio of objects. Note that if a multivariate predictive distribution for $y_p$ is given by $y_p|y\sim \mathcal{N} \left( \mu_p , \Xi_p\right)$, the predictive distribution of the sum is obtained as 
$$\mathbf{1}^T y_p\big|y\sim \mathcal{N} \left( \mathbf{1}^T\mu_{p} , \mathbf{1}^T\Xi_p\mathbf{1}\right),$$
where $\mathbf{1}$ is a vector of ones $\mathbf{1}=(1,\dots,1)^T$. In total, we predict $50$ times sums of different samples $y_p\in \mathbb{R}^{20}$ in every simulation run. Half of these $50$ samples for predicting sums are from the ``interpolation" and the ``extrapolation" test data each. Specifically, for both the ``interpolation" and the ``extrapolation" test data, we randomly select $25$ times disjoint sets of $20$ observations that are close together in space. We obtain these sets of $20$ close-by observations by randomly selecting a location and determining its $19$ nearest neighbors and then iteratively continuing in the same manner with the remaining locations until there are $25$ disjoints set for both the ``interpolation" and the ``extrapolation" test data. This is illustrated in Figure \ref{Train_test_locs}.

Further, we also measure the accuracy to learn the predictor function $F(\cdot)$ and predict the random effects $b$. This is done by evaluating predictions for $F(\cdot)$ and $b$ on the ``interpolation" test data sets. In addition, we evaluate the accuracy of estimates for the variance and covariance parameters $\theta$. For the independent deterministic boosting approaches which model spatial or grouped random effects using deterministic base learners (LSBoost, mboost, and CatBoost), no estimates for covariance parameters, the predictor function $F(\cdot)$, and the random effects $b$ can be obtained. The exception are grouped random effects where \texttt{mboost} reports estimated random effects $b$. This then also allows for obtaining an estimate for $F(\cdot)$ by assigning the test data a certain group from the training data and subtracting the value of the estimated random effect for this group from the obtained predictions. In the same way, we also obtain response variable predictions for data with new groups for mboost.

\subsection{Choice of tuning parameters}\label{tune_pars}
Tuning parameters are chosen using $4$-fold cross-validation on the training data in every simulation run by selecting the combination of tuning parameters from a full grid that minimizes the average mean square error on the validation data. For every boosting algorithm (LSBoost, mboost, CatBoost, GPBoost, and GPBoostOOS), we consider the following candidate tuning parameters: the number of boosting iterations $M\in \{1,\dots,1000\}$, the learning rate $\nu \in \{0.1,0.05,0.01\}$, the maximal tree depth $\in \{1,5,10\}$, and the minimal number of samples per leaf $\in \{1,10,100\}$. For the MERF algorithm, we choose the proportion of variables considered for making splits $\in \{0.5, 0.75, 1\}$. As in \citet{hajjem2014mixed}, we do not impose a maximal tree depth limit and set the number of trees to $300$. Note that the MERF algorithm implemented in the \texttt{MERF} package is very slow (see the results below) and choosing tuning parameters is thus computationally demanding. Generally, for random forests, the choice of tuning parameters is less important compared to boosting. For the \texttt{REEMtree} package, which relies on the \texttt{rpart} R package, trees are cost-complexity pruned and the amount of pruning is chosen using $10$-fold cross-validation on the training data.

\subsection{Results}\label{res_sim}
The results for the `hajjem' predictor function are reported in Table \ref{results_hajjem_one_way} for the grouped data and in Table \ref{results_hajjem_spatial} for the spatial data. The results for the other two predictor functions are reported in Appendix \ref{sim_res_oth}. In the tables, we report average values of the prediction accuracy measures over the simulation runs as well as corresponding standard errors. Further, we calculate p-values of paired two-sided t-tests comparing the GPBoost algorithm to the other approaches. Note that sometimes differences in accuracy measures are relatively small compared to standard errors, but p-values are highly significant nonetheless. The reason for this is that standard errors do not reflect correlation among methods over different simulation runs but paired t-tests do.\footnote{An example of this is the `RMSE\_new' in Table \ref{results_hajjem_one_way} when comparing GPBoost and CatBoost. The difference in the average `RMSE\_new' is relatively small compared to the corresponding standard errors, but the p-value is very small. In this case, GPBoost has a lower `RMSE\_new' than CatBoost in approximately $80$ of $100$ simulation runs, and the correlation between `RMSE\_new' of the two methods is approximately $0.9$ (results not tabulated).} For the covariance parameters, we report the RMSE and bias over the simulation runs. We also report the wall-clock time in seconds for training the different models. Note that the wall-clock times reported depend on the chosen tuning parameters. In particular, lower learning rates $\nu$ usually imply higher computational times for all boosting algorithms including the GPBoost algorithm.

\begin{table}[ht!]
\centering
\begingroup\footnotesize
\scalebox{0.9}{
\begin{tabular}{rlllllll|l}
  \hline
\hline
  & GPBoost & LinearME & LSBoost & CatBoost & mboost & MERF & REEMtree & GPBOOS \\ 
  \hline
RMSE & \bf{1.100} & 1.342 & 1.156 & 1.183 & 1.335 & 1.104 & 1.171 & 1.102 \\ 
  (SE) & (0.00144) & (0.00181) & (0.00205) & (0.00184) & (0.00215) & (0.00137) & (0.00163) & (0.00151) \\ 
  \lbrack p-val\rbrack &  & [4.66e-129] & [2.85e-59] & [3.45e-88] & [1.4e-112] & [1.31e-09] & [2.02e-83] & [8.06e-06] \\ 
   \hline
RMSE\_new & \bf{1.458} & 1.635 & 1.493 & 1.464 & 1.520 & 1.460 & 1.506 & 1.459 \\ 
  (SE) & (0.00292) & (0.00282) & (0.00294) & (0.00313) & (0.00285) & (0.00292) & (0.00298) & (0.00293) \\ 
  \lbrack p-val\rbrack &  & [6.43e-122] & [1.5e-51] & [2.09e-05] & [8.47e-86] & [6.08e-06] & [2.45e-78] & [6.45e-05] \\ 
   \hline
RMSE\_F & \bf{0.3370} & 0.8141 &  &  & 0.5495 & 0.3494 & 0.5111 & 0.3411 \\ 
  (SE) & (0.00243) & (0.00197) &  &  & (0.0024) & (0.00211) & (0.00227) & (0.00253) \\ 
  \lbrack p-val\rbrack &  & [2.37e-139] &  &  & [1.79e-102] & [4.13e-14] & [1.05e-91] & [2.59e-06] \\ 
  RMSE\_b & \bf{0.3193} & 0.3793 &  &  & 0.6934 & 0.3244 & 0.3363 & 0.3197 \\ 
  (SE) & (0.00109) & (0.00136) &  &  & (0.00221) & (0.00137) & (0.0012) & (0.00114) \\ 
  \lbrack p-val\rbrack &  & [8.1e-82] &  &  & [4.01e-121] & [1.55e-08] & [5.04e-46] & [0.0611] \\ 
   \hline
RMSE $\sigma^2_1$ & 0.07466 & 0.07812 &  &  &  & 0.07484 & 0.07402 & 0.07531 \\ 
  Bias $\sigma^2_1$ & 0.005512 & 0.005330 &  &  &  & 0.009059 & 0.003182 & 0.005787 \\ 
  RMSE $\sigma^2$ & 0.1703 & 0.6583 & 0.1361 & 0.07989 & 0.6286 & 0.1183 & 0.1477 & 0.1375 \\ 
  Bias $\sigma^2$ & -0.1672 & 0.6565 & -0.1172 & 0.07059 & 0.6266 & 0.1151 & 0.1436 & 0.1348 \\ 
   \hline
Time (s) & 0.9356 & 0.05323 & 0.3289 & 3.325 &   14.00 & 416.7 & 1.349 & 3.077 \\ 
   \hline
\hline
\end{tabular}
}
\endgroup
\caption{Results of the simulated experiments for the grouped data 
                          and the predictor function F = `hajjem'. For the prediction accuracy metrics for $y$, $F$, and $b$, averages over 
                      the simulation runs are reported. Corresponding standard errors are in parentheses.
                      P-values are calculated using paired t-tests comparing the GPBoost algorithm to the other approaches.
                      `GBPOOS' refers to the GPBoostOOS algorithm. Results for the test data with new groups are denoted by `\_new'. The smallest values are in boldface (excluding `GPBOOS'). 
                      An empty value indicates that the required predictions or estimates cannot be calculated.
                      Time refers to the average wall-clock time in seconds.} 
\label{results_hajjem_one_way}
\end{table}

Considering the results for the grouped data reported in Table \ref{results_hajjem_one_way}, we find that the GPBoost algorithm significantly outperforms all other methods in all prediction accuracy measures including the ones concerning the learning, or prediction, of both the predictor function $F(\cdot)$ and the random effects $b$. Apart from the GPBoostOOS algorithm, the MERF algorithm has the second highest prediction accuracy. Gradient boosting with a square loss including the grouping variable as a categorical variable using the approach of \texttt{LightGBM} (LSBoost), CatBoost, and mboost all have significantly lower prediction accuracy compared to the GPBoost algorithm for data with both existing groups (`RMSE') and new groups (`RMSE\_new'). LSBoost, i.e., \texttt{LightGBM}, performs better than CatBoost for predicting data with existing groups and the opposite holds for new groups. Apparently, LSBoost can learn group effects, i.e., the effect of a high-cardinality categorical variable, better than CatBoost, but CatBoost learns the predictor function $F(\cdot)$ for the non-categorical predictor variables $X$ better than LSBoost. The reason for the former is unclear to us. The likely reason for the latter finding is that LSBoost has higher variance and does some overfitting due to the high-cardinality categorical variable, and CatBoost mitigates this due to its ordered boosting mode. Not surprisingly, a linear mixed effects model (LinearME) performs considerably worse than the GPBoost algorithm in all prediction accuracy measures and the estimation of the predictor function $F(\cdot)$.

Concerning variance parameter estimates, we observe no major differences among the methods in the RMSE of the variance of the random effects $\sigma^2_1$. For the error variance $\sigma^2$, we observe large differences, though. In particular, all methods have biased estimates for the error variance parameter. As expected, the linear model has an upward bias which compensates for the misspecification of $F(\cdot)$, and the GPBoost algorithm has a downward bias. The latter finding is in line with the recent observation that state-of-the-art machine learning methods can interpolate the training data while at the same time having a low generalization error as discussed in Section \ref{gpboostoost}. When estimating the covariance parameters on out-of-sample data using the GPBoostOOS algorithm, the RMSE of $\sigma^2$ is smaller and there is no downward bias anymore. Concerning computational time, not surprisingly, the linear model has the lowest computational time. LSBoost is approximately three times faster than GPBoost, and CatBoost is approximately three times slower than GPBoost. Further, the GPBoost algorithm runs approximately one order of magnitude faster than mboost and several orders of magnitude faster than the MERF algorithm.

For the `friedman3' predictor function for the grouped data, we find qualitatively similar results as for the `hajjem' predictor function; see Table \ref{results_friedman3_one_way} in the appendix. Further, as expected, a linear mixed effects model performs best in the case where the true predictor function is linear; see Table \ref{results_linear_one_way} in the appendix. The differences between the linear model and the GPBoost algorithm are of small, albeit significant, magnitude. Except for the linear model, the GPBoost algorithm significantly outperforms all other approaches when $F(\cdot)$ is a linear function. The MERF algorithm has a relatively high RMSE for the predictor function $F(\cdot)$ which also translates to low prediction accuracy for the response variable for data with both existing groups (`RMSE') and new groups (`RMSE\_new'). For the `friedman3' and `linear' predictor functions, we again observe that the GPBoostOOS algorithm has a lower RMSE for the error variance $\sigma^2$ compared to the GPBoost algorithm and no downward bias. 

\begin{table}[ht!]
\centering
\begingroup\footnotesize
\scalebox{0.9}{
\begin{tabular}{rlllll|l}
  \hline
\hline
  & GPBoost & LinearGP & LSBoost & mboost & TwoStep & GPBOOS \\ 
  \hline
RMSE & \bf{1.374} & 1.466 & 1.474 & 1.479 & 1.406 & 1.372 \\ 
  (SE) & (0.00574) & (0.00567) & (0.0067) & (0.00633) & (0.00625) & (0.00556) \\ 
  \lbrack p-val\rbrack &  & [1.83e-45] & [2.33e-35] & [8.74e-42] & [1.87e-12] & [0.329] \\ 
  CRPS & \bf{0.8011} & 0.8183 & 0.8807 & 0.8280 & 0.8043 & 0.7683 \\ 
  (SE) & (0.00564) & (0.00295) & (0.00673) & (0.00354) & (0.00448) & (0.00305) \\ 
  \lbrack p-val\rbrack &  & [0.000359] & [1.09e-22] & [5.02e-07] & [0.484] & [1.23e-10] \\ 
   \hline
RMSE\_ext & \bf{1.519} & 1.599 & 1.611 & 1.899 & 1.536 & 1.517 \\ 
  (SE) & (0.00978) & (0.00896) & (0.0164) & (0.0377) & (0.00963) & (0.00928) \\ 
  \lbrack p-val\rbrack &  & [2.8e-33] & [3.91e-11] & [3.08e-16] & [5.15e-05] & [0.551] \\ 
  CRPS\_ext & \bf{0.8689} & 0.8960 & 0.9807 & 1.101 & 0.8808 & 0.8526 \\ 
  (SE) & (0.00686) & (0.00519) & (0.0131) & (0.0261) & (0.00626) & (0.00535) \\ 
  \lbrack p-val\rbrack &  & [7.15e-11] & [1.92e-17] & [1.32e-13] & [0.000694] & [9.99e-06] \\ 
   \hline
RMSE\_sum & \bf{11.70} & 11.90 & 13.80 & 18.44 & 11.90 & 11.62 \\ 
  (SE) & (0.211) & (0.204) & (0.328) & (0.697) & (0.206) & (0.206) \\ 
  \lbrack p-val\rbrack &  & [0.00573] & [1.14e-12] & [4.05e-15] & [0.00416] & [0.151] \\ 
  CRPS\_sum & \bf{6.468} & 6.509 & 8.508 & 10.45 & 6.710 & 6.331 \\ 
  (SE) & (0.115) & (0.103) & (0.22) & (0.408) & (0.117) & (0.102) \\ 
  \lbrack p-val\rbrack &  & [0.403] & [5.92e-19] & [1.69e-15] & [2.69e-06] & [0.000978] \\ 
   \hline
RMSE\_F & \bf{0.7125} & 0.8741 &  & 0.7538 & 0.7538 & 0.7102 \\ 
  (SE) & (0.00869) & (0.00692) &  & (0.00854) & (0.00854) & (0.00817) \\ 
  \lbrack p-val\rbrack &  & [1.17e-49] &  & [4.4e-13] & [4.4e-13] & [0.577] \\ 
  RMSE\_b & \bf{0.7080} & 0.7091 &  & 0.7285 & 0.7285 & 0.7034 \\ 
  (SE) & (0.00729) & (0.0065) &  & (0.00714) & (0.00714) & (0.00696) \\ 
  \lbrack p-val\rbrack &  & [0.717] &  & [6.17e-07] & [6.17e-07] & [0.0609] \\ 
   \hline
RMSE $\sigma^2_1$ & 0.2810 & 0.2690 &  &  & 0.4950 & 0.2597 \\ 
  Bias $\sigma^2_1$ & -0.004313 & -0.006602 &  &  & -0.4160 & 0.005215 \\ 
  RMSE $\rho$ & 0.08537 & 0.03392 &  &  & 0.03529 & 0.03221 \\ 
  Bias $\rho$ & 0.008375 & -0.004949 &  &  & -0.002940 & -0.004969 \\ 
  RMSE $\sigma^2$ & 0.5672 & 0.6185 & 0.4955 & 0.9024 & 0.3251 & 0.4975 \\ 
  Bias $\sigma^2$ & -0.4539 & 0.5778 & -0.2283 & 0.8832 & -0.1561 & 0.4559 \\ 
   \hline
Time (s) & 27.30 & 0.5509 & 0.09183 & 3.204 & 0.5057 & 36.12 \\ 
   \hline
\hline
\end{tabular}
}
\endgroup
\caption{Results of the simulated experiments for the spatial data 
                          and the predictor function F = `hajjem'. For the prediction accuracy metrics for $y$, $F$, and $b$, averages over 
                      the simulation runs are reported. Corresponding standard errors are in parentheses.
                      P-values are calculated using paired t-tests comparing the GPBoost algorithm to the other approaches.
                      `GBPOOS' refers to the GPBoostOOS algorithm. Results for the ``extrapolation" test data are denoted by `\_ext', and results for the predictions of sums are denoted by `\_sum'. The smallest values are in boldface (excluding `GPBOOS'). 
                      An empty value indicates that the required predictions or estimates cannot be calculated.
                      Time refers to the average wall-clock time in seconds.} 
\label{results_hajjem_spatial}
\end{table}

We next discuss the results for the spatial data reported in Table \ref{results_hajjem_spatial}. We find that the GPBoost algorithm has higher prediction accuracy compared to all alternative approaches in all measures with most of the differences being highly significant. In particular, the GPBoost algorithm has higher prediction accuracy in terms of the RMSE and CRPS compared to both a linear Gaussian process model with $F(X)=X^T\beta$ and independent boosting including the coordinates in the predictor function $F(\cdot)$. A two-step approach performs better than a linear model and independent boosting in the majority of prediction accuracy measures, but it has lower prediction accuracy compared to the GPBoost algorithm. The results for the `friedman3' predictor function reported in Table \ref{results_friedman3_spatial} in the appendix are similar. The GPBoost algorithm significantly outperforms all other approaches in all prediction accuracy measures. Not surprisingly, a linear model has the highest prediction accuracy when the data generating process is linear; see Table \ref{results_linear_spatial} in the appendix. However, despite the relatively small sample size of $n=500$, the differences in prediction accuracy between the linear Gaussian process model and the GPBoost algorithm are relatively small. As expected due to the arguments laid out in Section \ref{exist_work}, the differences between the GPBoost algorithm and independent boosting including the coordinates in the predictor variables (LSBoost) are the largest for probabilistic predictions of sums (`CRPS\_sum') for all three predictor functions.


\section{Real-world Applications}\label{data_appl}
In the following, we apply the GPBoost algorithm to several real-world data sets and compare its prediction accuracy to alternative methods. We consider grouped data, spatial data, and several UCI benchmark data sets. 

\subsection{Grouped random effects for categorical data: wages data}
We first apply the GPBoost algorithm to a grouped data set, i.e., a data set with a high-cardinality categorical predictor variable. We use panel data from the National Longitudinal Survey of Young Working Women consisting of $28'534$ observations for $4'711$ young working women. This data was collected within the ``National Longitudinal Survey" over the years 1968-1988, and it can be downloaded from \url{https://www.stata-press.com/data/r10/nlswork.dta}. The response variable is the logarithmic real wage, the persons ID number constitutes the high-cardinality categorical grouping variable, and the data includes the following predictor variables: \texttt{age}, \texttt{ttl\_exp} (total work experience), \texttt{tenure} (job tenure in years), \texttt{not\_smsa} (1 if not SMSA), \texttt{south} (1 if south), \texttt{year} (interview year), \texttt{msp} (1 if married, spouse present), \texttt{nev\_mar} (1 if never married), \texttt{collgrad} (1 if college graduate), \texttt{c\_city} (1 if central city), \texttt{hours} (usual hours worked), \texttt{grade} (current grade completed), \texttt{ind\_code} (industry of employment), \texttt{occ\_code} (occupation), and \texttt{race} (1=white, 2=black, 3=other). The low-cardinality categorical variables \texttt{ind\_code}, \texttt{occ\_code}, \texttt{race}, and \texttt{year} are dummy coded. Further, we also include the square of \texttt{age}, \texttt{ttl\_exp}, and \texttt{tenure} in the linear model.

We compare the prediction accuracy of different approaches using nested $4$-fold cross-validation. Specifically, all observations are partitioned into four disjoint sets, and, in every fold, one of the sets is used as test data and the remaining data is used for training. Note that the test data sets contain both groups that are observed and unobserved in the training data. We compare the GPBoost algorithm to the same alternative methods as in the simulation study; see Section \ref{methods}. Concerning independent gradient boosting, we also report the results when including the high-cardinality categorical variable as a continuous variable (`LSBoostCont') in addition to including it as a categorical variable using the approach of \texttt{LightGBM} (`LSBoostCat'). Tuning parameters are chosen by doing an additional inner $4$-fold cross-validation in the same manner as the outer cross-validation on every of the four training data sets using the mean square error as selection criterion.\footnote{Note that one has to be careful when doing cross-validation for dependent data to avoid biased estimates of the generalization error as pointed out by, e.g., \citet{rabinowicz2020cross}. However, apart from the fact that the validation and test data sets are of slightly different sizes, our cross-validation setting preserves the distributional relation between the inner fold training and validation data sets and the training and test data sets and, consequently, no bias is introduced.} We consider the same set of tuning parameters as in the simulated experiments; see Section \ref{tune_pars}.

The results are summarized in Table \ref{results_wages} and Figure \ref{results_per_fold_wages}. Table \ref{results_wages} reports the average test RMSE over all folds, and Figure \ref{results_per_fold_wages} graphically displays the test RMSE per fold. We find that the GPBoost algorithm has higher prediction accuracy compared to all other methods on average over all folds and also for every fold separately. The second-best method is the MERF algorithm, and CatBoost has the third-lowest test RMSE. In contrast to the simulated experiments but as suggested by the authors of \texttt{LightGBM}\footnote{\url{https://lightgbm.readthedocs.io/en/latest/Advanced-Topics.html\#categorical-feature-support} (retrieved on August 22, 2022)}, considering the grouping variable as a continuous variable (`LSBoostCont') results in higher prediction accuracy compared to the approach of \texttt{LightGBM} for categorical variables (`LSBoostCat').

\begin{table}[ht!]
\centering
\begingroup\footnotesize
\begin{tabular}{rlllllll}
  \hline
\hline
GPBoost & LinearME & LSBoostCat & LSBoostCont & CatBoost & mboost & MERF & REEMtree \\ 
  \hline
\bf{0.296} & 0.305 & 0.321 & 0.313 & 0.303 & 0.331 & 0.299 & 0.322 \\ 
   \hline
\hline
\end{tabular}
\endgroup
\caption{Average test RMSE for the wages data.} 
\label{results_wages}
\end{table}

\begin{figure*}[ht!]
	\centering
	\includegraphics[width=0.6\textwidth]{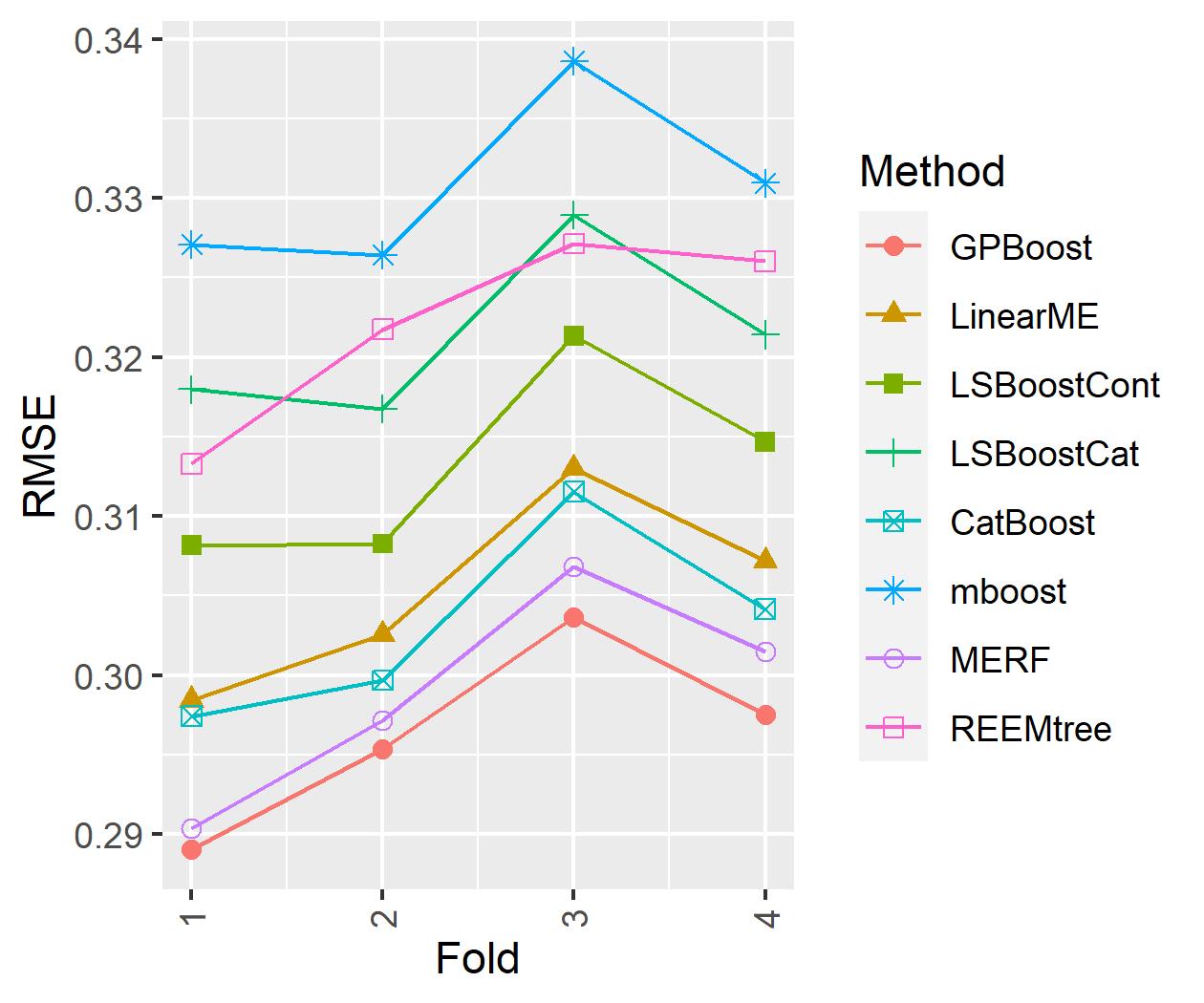}
	\caption{Test RMSE for the wages data for every fold separately.}
	\label{results_per_fold_wages} 
\end{figure*}

\subsection{Gaussian processes for spatial data: house price data}\label{house_data}
We next apply the GPBoost algorithm to a spatial data set. We use house price data for $25'357$ single-family homes sold in Lucas County, Ohio. This data is available in the \texttt{spData} R package \citep{bivand2008applied}, and it has been previously studied by \citet{lesage2004models,bivand2011after,dube2013dealing}. The response variable is the logarithmic selling price, and the data includes the following predictor variables: \texttt{age}, \texttt{stories} (factor with levels \{one, bilevel, multilvl, one+half, two, two+half, three\}), \texttt{TLA} (total living area), \texttt{wall} (factor with levels \{stucdrvt, ccbtile, metlvnyl, brick, stone, wood partbrk\}), \texttt{beds} (number of bedrooms), \texttt{baths} (number of full baths), \texttt{halfbaths} (number of halfhbaths), \texttt{frontage} (lot frontage), \texttt{depth}, \texttt{garage} (factor with levels \{no garage, basement, attached, detached, carport\}), \texttt{garagesqft}, \texttt{rooms} (number of rooms), \texttt{lotsize}, \texttt{sdate} (year in which the house was sold, $1993\leq$ \texttt{sdate} $\leq 1998$), as well as longitude-latitude coordinates for the location. For the Gaussian process model with a linear predictor function, we follow \citet{bivand2011after} and also include the square and cube of age as predictor variables, and as in \citet{dube2013dealing}, we logarithmize the total living area and the lot size. The left plot of Figure \ref{plot_house_data} shows the observation locations and observed logarithmic prices. Further, the right plot of Figure \ref{plot_house_data} shows smoothed differences of log-prices from the global mean. The latter are obtained by estimating a zero-mean Gaussian process to the differences of the log-prices from the global mean and calculating the posterior mean.
\begin{figure*}[ht!]
	\centering
	\includegraphics[width=0.49\textwidth]{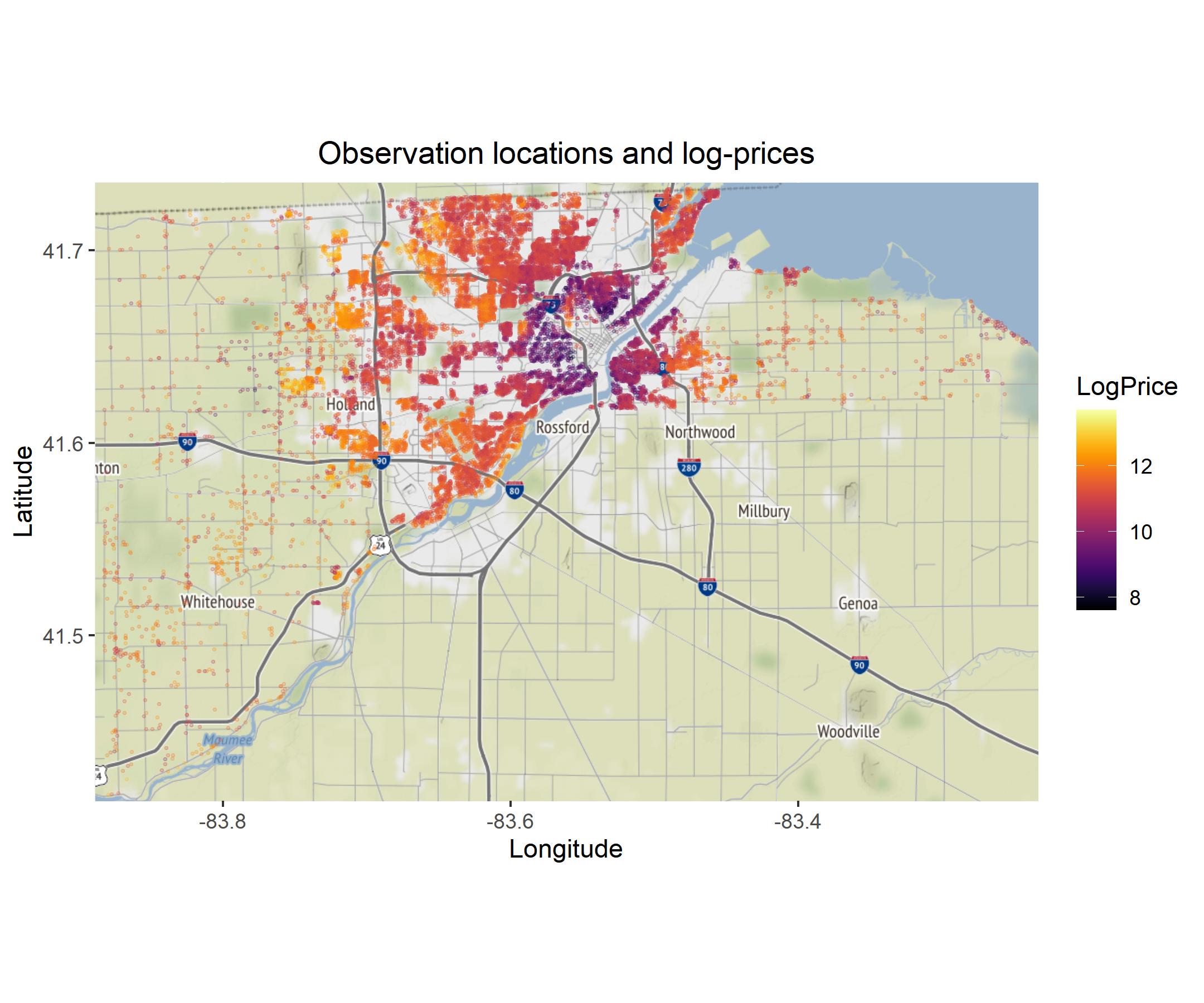}
	\includegraphics[width=0.49\textwidth]{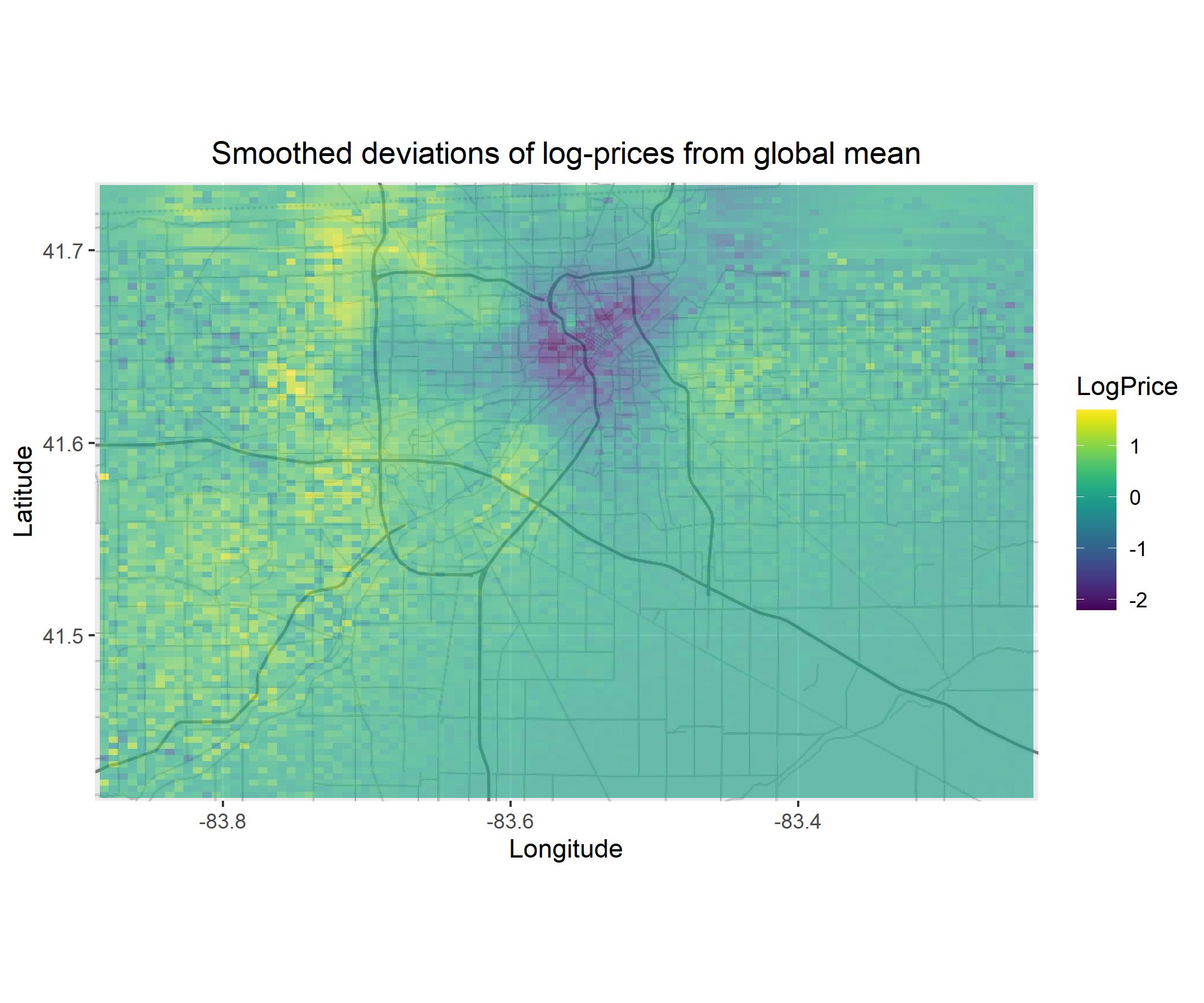}
	\caption{Illustration of house price data: map with observation locations and log-prices (left plot) and smoothed differences of log-prices from the global mean (right plot).}
	\label{plot_house_data} 
\end{figure*}

We compare the GPBoost algorithm to the same approaches as in the simulation study: a linear Gaussian process model with $F(X)=X^T\beta$, independent gradient boosting, mboost with spatial spline base learners, and a two-step approach; see Section \ref{methods} for more information. For independent boosting and the GPBoost algorithm, we include the coordinates as predictor variables in the predictor function $F(\cdot)$. For comparison, we also report results when excluding the coordinates from the predictor function. For the Gaussian process model and the GPBoost algorithm, we use an exponential covariance function. Due to the relatively large sample size, we use a Vecchia approximation as outlined in Sections \ref{largedata} and \ref{largedata2} for the Gaussian process-based models. Specifically, for training, we use a Vecchia approximation for the response variable with $50$ nearest neighbors and a random ordering of the observations; see Section \ref{vecchia_resp}. For prediction, we use the result in Proposition \ref{PredVEcchiaOF} with the observed data ordered first, conditioning on observed data only when calculating the Vecchia approximation, and using $500$ nearest neighbors.\footnote{We set the number of nearest neighbors for prediction to the largest value that our computational budget supports. Using a smaller number of nearest neighbors decreases the prediction accuracy. Whether a higher prediction accuracy can be achieved with an even larger number of nearest neighbors is unclear.} 

Prediction accuracy is evaluated by partitioning the data into expanding window training data sets and temporal out-of-sample test data sets. Specifically, learning is done on an expanding window containing all data up to the year $t-1$, and predictions are calculated for the next year $t$. We use the three years $t\in\{1996,1997,1998\}$ as test data. For every $t$, we additionally split the training data into two subsets: inner training data containing all data up to year $t-2$ and validation data for the year $t-1$. Tuning parameters of the boosting methods are chosen separately for $t\in\{1996,1997,1998\}$ by learning on the inner training data and minimizing the mean square error on the validation data. We consider the same grid of tuning parameters as in the simulated experiments; see Section \ref{tune_pars}. 

In addition to univariate predictions, we also generate predictions for the total value of multiple houses as explained in the following. For every test set, we randomly select $100$ times disjoint sets of $20$ observations that are close-by in space. Such sets of spatially close samples are determined by first randomly selecting a sample, finding its $19$ nearest neighbors, and then iteratively continuing in the same manner with the remaining locations. For these sets of $20$ objects, we calculate predictive distributions for their sums as described in Section \ref{eval_crit}.\footnote{For simplicity, we predict the sum on the logarithmic scale. If the sum should be predicted on the original scale, simulation is required as the sum of dependent log-normal variables does not follow a standard distribution.} For generating multivariate predictive distributions of dimension $20$ required for the latter, we also use the Vecchia approximation as outlined in Proposition \ref{PredVEcchiaOF} with the observed data ordered first, but we condition on all data and not just the observed data. The latter is computationally more expensive, but it allows for obtaining non-diagonal and more accurate predictive covariance matrices; see Section \ref{largedata2} for more information.

Further, to disentangle the effect of the Vecchia approximation and the standard GPBoost algorithm without an approximation for large data, we also run the experiments on smaller subsets of the full data. This allows to do all Gaussian process-related calculations exactly. Such smaller-sized subsets are obtained as follows. For every year, we randomly partition the data into ten disjoint equal-sized subsets. Every such subset is used as a training data set, and for every training data set, we use another one of the ten subsets as test set such that each subset is a test set once. In doing so, we obtain $60$ training and test data sets with an average sample size of approximately $400$. For these small subsets, tuning parameters are chosen using $4$-fold cross-validation on every training data set.

As in the simulated experiments in Section \ref{simul}, we measure the accuracy of both point predictions and probabilistic predictions using the RMSE and the CRPS. Further, we also evaluate the accuracy of $\alpha-$quantile predictions for $\alpha=0.05$. We do this since, in practice, predictions for lower quantiles are required for risk management purposes such as calculating a value-at-risk. Quantile predictions are obtained by assuming Gaussian predictive distributions with means and variances determined by the fixed and random effects. If a method does not contain random effects, we use the residual variance as in the simulated experiments. Quantile predictions are evaluated using the quantile loss \citep{gneiting2007probabilistic}
\begin{equation}\label{quant_loss}
S(y_s,\hat y_\alpha) = (y_s-\hat y_\alpha) (\alpha-\mathbbm{1}_{\{y_s\leq\hat y_\alpha \}}),
\end{equation}
where $\hat y_\alpha$ denotes the predicted quantile and $y_s$ the observed data. 

Table \ref{results_housing} reports the prediction accuracy results averaged over the different test sets, Figure \ref{results_per_fold_housing} graphically displays the results per test set for the full data, and Figure \ref{results_per_fold_housing_small} illustrates the results of the small subsets over different folds. Overall, the GPBoost algorithm has the highest prediction accuracy. Specifically, we find that the GPBoost algorithm has higher prediction accuracy compared to a linear Gaussian process model, the mboost algorithm, and a two-step approach in all prediction accuracy measures. This holds true for both the full data set and the small random subsets for which exact Gaussian process inference is carried out. For univariate predictions (`RMSE', `CRPS', and `QL'), the GPBoost algorithm also has almost the same prediction accuracy as independent boosting (`LSBoost'). For the prediction of sums of several objects and the small subsets, the GPBoost algorithm has considerably higher prediction accuracy compared to independent boosting in all measures (`RMSE\_sum', `CRPS\_sum', `QL\_sum', `RMSE\_small', `CRPS\_small', and `QL\_small').

\begin{table}[ht!]
\centering
\begingroup\footnotesize
\begin{tabular}{rlllll|ll}
  \hline
\hline
  & GPBoost & LinearGP & LSBoost & mboost & TwoStep & GPBoost\_ex & LSBoost\_ex \\ 
  \hline
RMSE & \bf{0.266} & 0.353 & 0.266 & 0.332 & 0.306 & 0.290 & 0.338 \\ 
  CRPS & 0.144 & 0.176 & \bf{0.143} & 0.174 & 0.162 & 0.155 & 0.180 \\ 
  QL & \bf{0.0358} & 0.0516 & 0.0360 & 0.0468 & 0.0384 & 0.0370 & 0.0461 \\ 
   \hline
RMSE\_sum & \bf{1.72} & 2.50 & 1.94 & 3.14 & 2.31 & 2.17 & 3.73 \\ 
  CRPS\_sum & \bf{1.07} & 1.34 & 1.23 & 1.85 & 1.38 & 1.29 & 2.30 \\ 
  QL\_sum & \bf{0.185} & 0.446 & 0.234 & 0.517 & 0.320 & 0.263 & 0.705 \\ 
   \hline
RMSE\_small & \bf{0.307} & 0.331 & 0.341 & 0.358 & 0.334 & 0.307 & 0.368 \\ 
  CRPS\_small & \bf{0.162} & 0.171 & 0.189 & 0.189 & 0.179 & 0.162 & 0.196 \\ 
  QL\_small & \bf{0.0402} & 0.0431 & 0.0611 & 0.0489 & 0.0470 & 0.0400 & 0.0554 \\ 
   \hline
\hline
\end{tabular}
\endgroup
\caption{Average prediction accuracy measures for the housing data. 
      Results for the predictions of sums are denoted by `\_sum'. 
      Results for the small subsampled data sets are denoted by `\_small'.
      `QL' denotes the quantile loss.
      `\_ex' denotes results when not including the coordinates in 
                   the variables for the predictor function.} 
\label{results_housing}
\end{table}

The comparison between the GPBoost algorithm when not including the locations in the predictor function $F(\cdot)$ (`GPBoost\_ex') and the corresponding independent boosting version (`LSBoost\_ex') shows that there is important residual spatial variation also conditional on the effect of the predictor variables. Considering the results for the full data, the GPBoost algorithm with the coordinates included in the predictor function (`GPBoost') performs better compared to the GPBoost algorithm when the coordinates are not included in the predictor function (`GPBoost\_ex'). This is an indication that there are interaction effects between the predictor variables and the spatial locations. For the small data, the results of `GPBoost' and `GPBoost\_ex' are almost equivalent. A possible explanation for this is that the sample size is too small to allow for modeling such interaction effects.
\begin{figure*}[ht!]
	\centering
	\includegraphics[width=\textwidth]{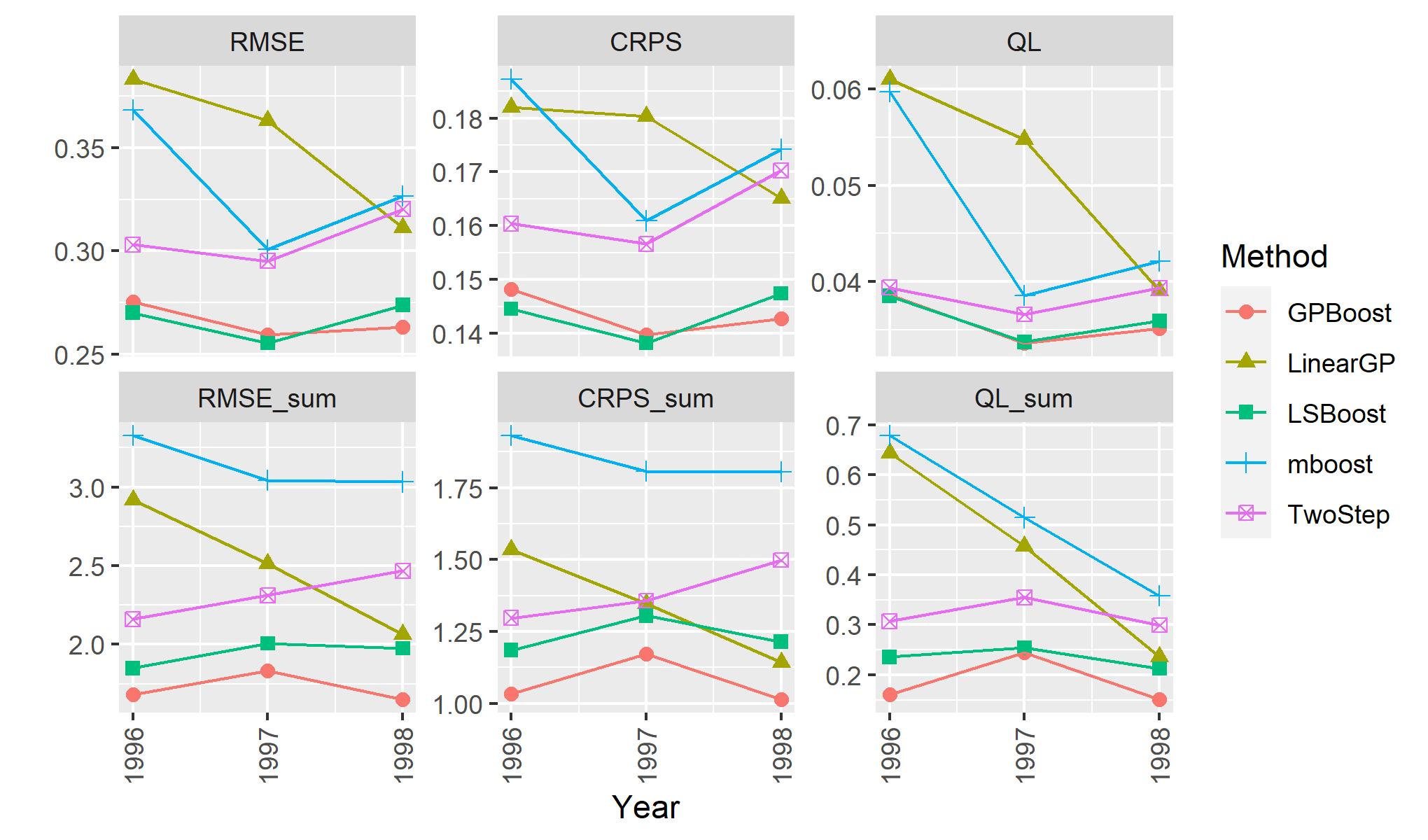}
	\caption{Prediction accuracy for the housing data for every fold separately. Results for the predictions of sums are denoted by `\_sum'. `QL' denotes the quantile loss.}
	\label{results_per_fold_housing} 
\end{figure*}

\begin{figure*}[ht!]
	\centering
	\includegraphics[width=0.8\textwidth]{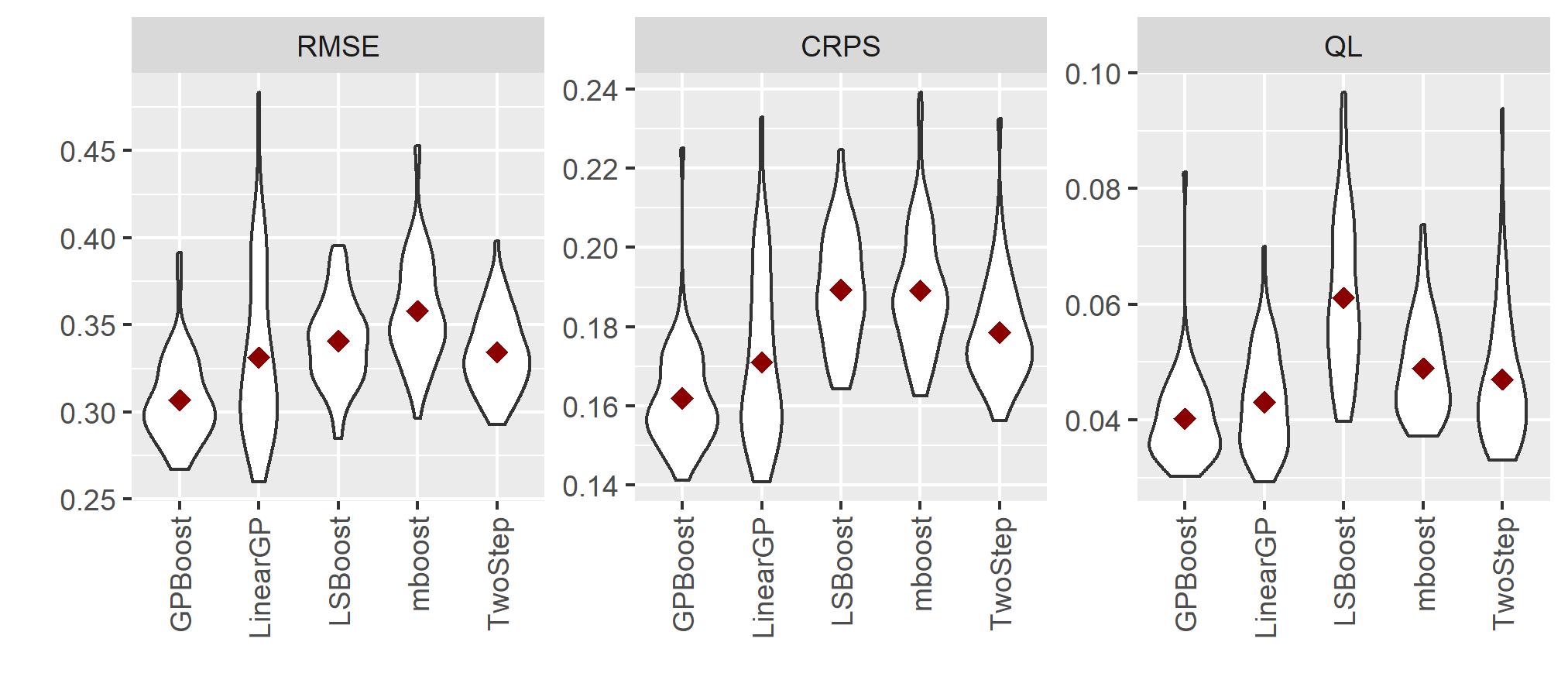}
	\caption{Violin plots illustrating the prediction accuracy for the small subsets of the housing data. The red rhombi represent means over the sample splits.}
	\label{results_per_fold_housing_small} 
\end{figure*}

For illustration, we plot in Figure \ref{pred_house_price} predictive posterior means and variances of the Gaussian process part of the GPBoost model applied to the entire data set when not including the locations in the fixed effects function. We use the tuning parameters obtained on the last validation data set. Comparing the predicted mean field to the one of Figure \ref{plot_house_data}, we see that in some areas the spatial effect is different when factoring out the effect of the other predictor variables. As expected, predictive variances are high in areas with few observations.

\begin{figure*}[ht!]
	\centering
	\includegraphics[width=0.49\textwidth]{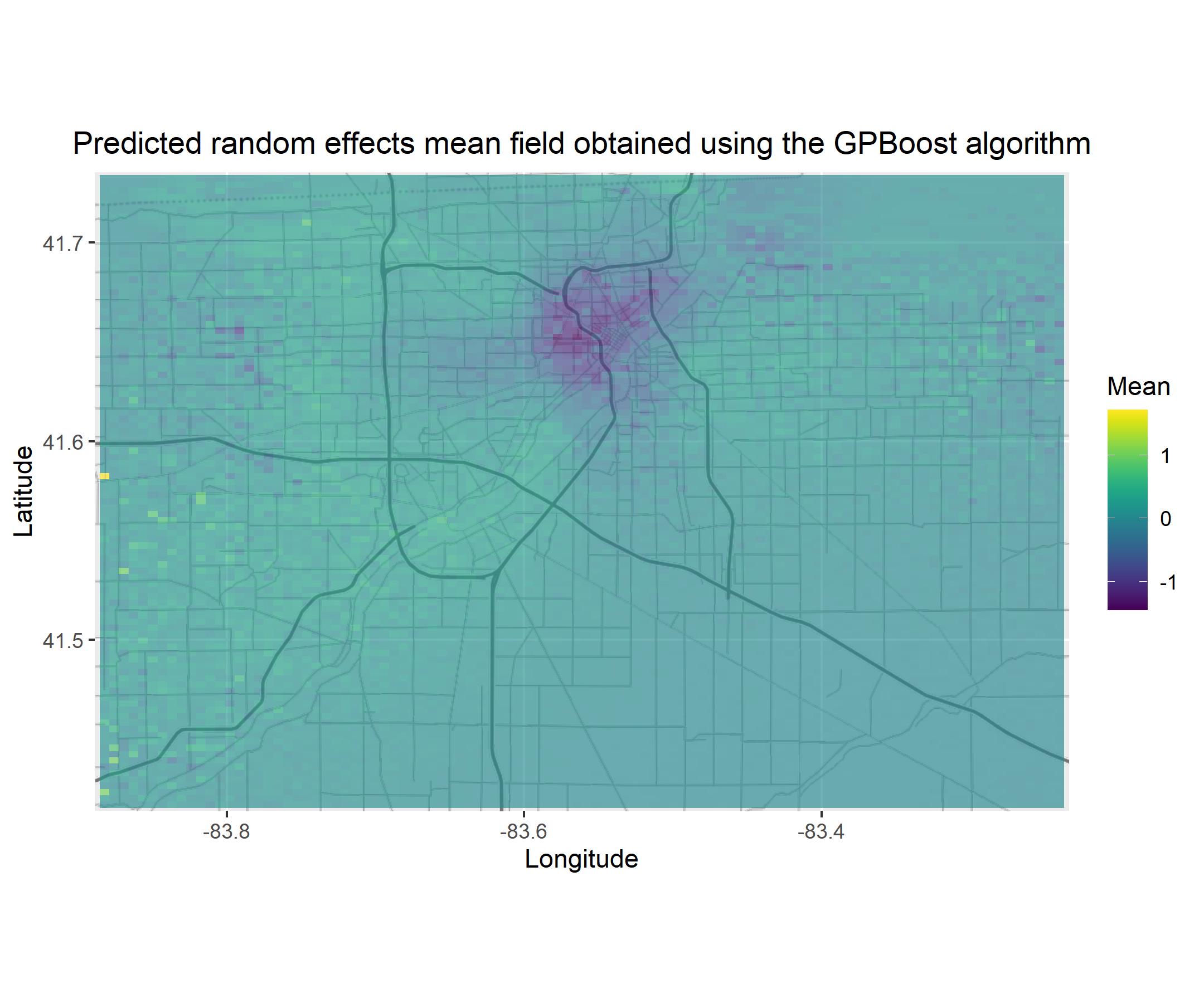}
	\includegraphics[width=0.49\textwidth]{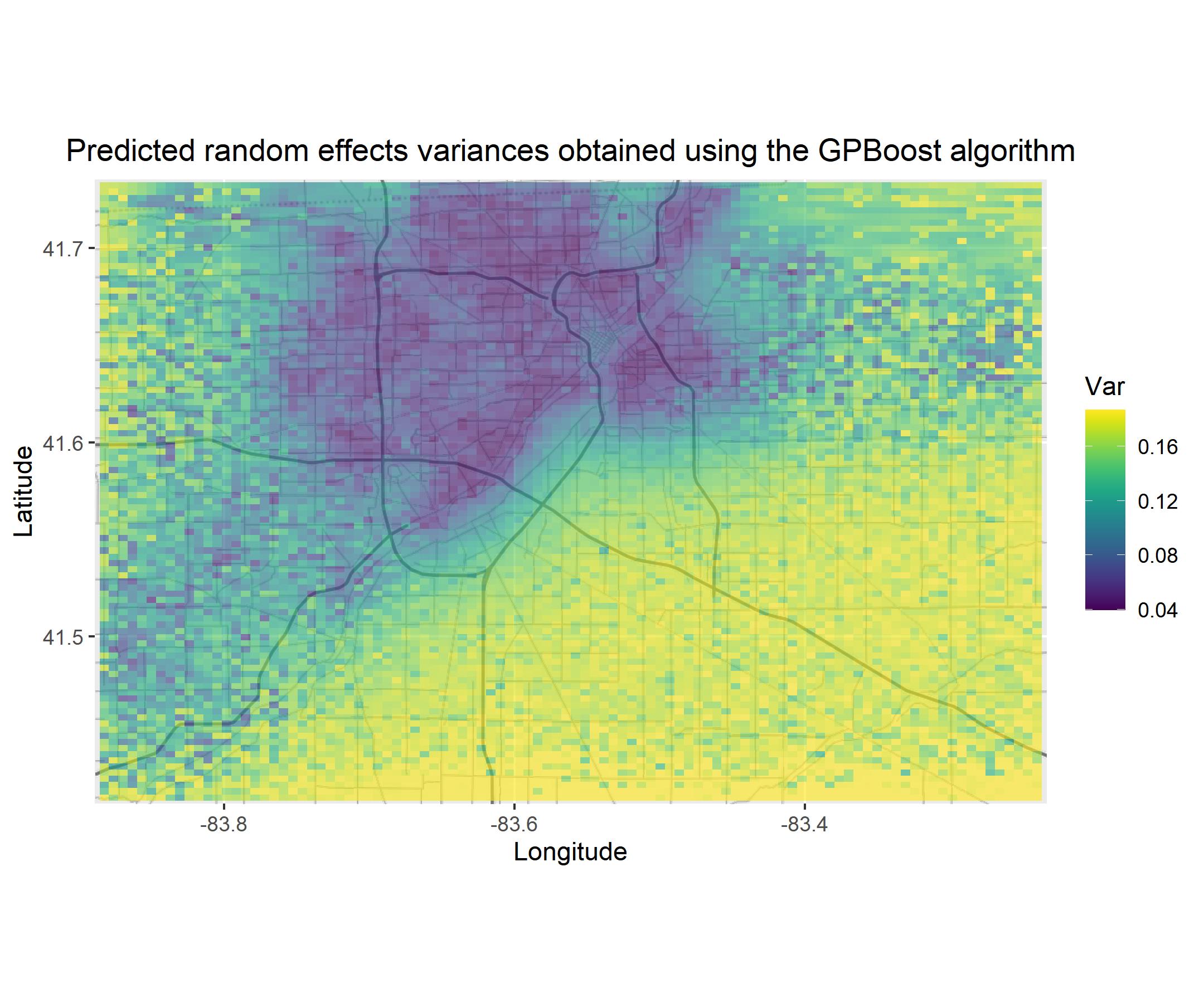}
	\caption{Predicted random effects means and variances obtained using the GPBoost algorithm when not including the locations in the fixed effects function.}
	\label{pred_house_price} 
\end{figure*}

\subsection{Gaussian processes for non-spatial data: UCI benchmark data sets}
In spatial statistics, one usually distinguishes between locations $S$ and other predictor variables $X$. The locations $S$ are used as input features for a spatial Gaussian process and the effect of $X$ is modeled using a predictor function $F(\cdot)$. This is the way we have used Gaussian processes so far in this paper. In machine learning, Gaussian processes are often used to model non-spatial data by including non-spatial predictor variables as input features for Gaussian process models. 

In the following, we compare the GPBoost algorithm to independent boosting and Gaussian process regression using several UCI data set repository\footnote{\url{http://archive.ics.uci.edu/ml/index.php}} benchmark data sets. In doing so, all predictor variables are used as features $X$ in $F(\cdot)$ for boosting and as input features $S$ in the Gaussian process models. In particular, for the GPBoost algorithm, the predictor variables are used both as features $X$ for the function $F(\cdot)$ and as input features for the Gaussian process part, and $X$ and $S$ are equal. Table \ref{data_sum} shows the data sets. All data sets have a relatively small sample size such that the Gaussian process-related calculations can be done exactly without using an approximation for large data. The analysis of the GPBoost algorithm for larger benchmark data sets with different approximations for large (non-spatial) data is left for future research.

\begin{table}[ht!]
\centering
\begingroup\footnotesize
\begin{tabular}{rrrrrrrrrr}
\hline\hline
& fertility & servo & machine & yacht & autompg & boston & pendulum & energy & concrete\\ 
\hline
$n$: & 100 & 167 & 209 & 308 & 392 & 506 & 630 & 768 & 1030 \\
$p = d$: & 9 & 4 & 7 & 6 & 7 & 13 & 9 & 8 & 8 \\
\hline\hline
\end{tabular}
\endgroup
\caption{Overview of data sets.}
\label{data_sum}
\end{table}

We use $10$-fold cross-validation to analyze the prediction accuracy. All input features $S$ for the Gaussian process models are standardized to have mean zero and standard deviation one, where the parameters for the standardization are calculated on the training data only. Tuning parameters for independent tree-boosting and the GPBoost algorithm are chosen by doing an additional inner $4$-fold cross-validation on every training data set. We use the same set of candidate tuning parameters as in all experiments above and an exponential covariance function. 

In the machine learning Gaussian process literature, probabilistic predictions are often evaluated using the average (univariate Gaussian) test negative log-likelihood (NLL). We also include the average test NLL given by
$$\frac{1}{n_p}\sum_{i=1}^{n_p}0.5(y_{p,i}-\mu_{p,i})^2/\sigma_{p,i}^2 + 0.5\log(\sigma_{p,i}^22\pi) $$
 as prediction accuracy measure. The results shown in the following for the test NLL are obtained by also using the test NLL as criterion for choosing tuning parameters in the inner 4-fold cross-validation. 

Table \ref{results_UCI_data} shows average prediction accuracy measures for the different methods. In parentheses are standard errors. Note that such standard errors have to be taken with a grain of salt.\footnote{First, the standard errors are calculated assuming independence of the sample splits which likely does not hold true given that the training data sets for the different sample splits have substantial overlap. Further, similarly as mentioned for the simulated experiments in Section \ref{res_sim}, such univariate standard errors do not reflect correlation among methods across different sample splits. The main reason we report such standard errors is that this is common practice in the machine learning literature.} We observe that the GPBoost algorithm has higher prediction accuracy compared to independent boosting and Gaussian process regression for the majority of the data sets and prediction accuracy measures. Specifically, the GPBoost algorithm has an average rank of approximately $1.33$ for the test RMSE and $1.22$ for the CRPS and the test NLL. Further, in the few cases where either independent boosting or Gaussian process regression has the highest prediction accuracy, the GPBoost algorithm always ranks second.

\begin{table}[ht!]
\centering
\begingroup\footnotesize
\begin{tabular}{rlll}
  \hline
\hline
  & GPBoost & GP & LSBoost \\ 
  \hline
\textbf{RMSE} &  &  &  \\ 
  fertility & 0.171 (0.011) & 0.182 (0.0092) & \textbf{0.166} (0.011) \\ 
  servo & \textbf{0.270} (0.020) & 0.281 (0.020) & 0.288 (0.026) \\ 
  machine & \textbf{0.342} (0.013) & 0.368 (0.014) & 0.347 (0.012) \\ 
  yacht & \textbf{0.492} (0.077) & 2.47 (0.35) & 0.686 (0.090) \\ 
  autompg & 2.53 (0.18) & \textbf{2.52} (0.19) & 2.79 (0.24) \\ 
  boston & \textbf{2.81} (0.22) & 2.96 (0.22) & 3.26 (0.19) \\ 
  pendulum & \textbf{1.62} (0.14) & 1.64 (0.14) & 2.23 (0.18) \\ 
  energy & 0.393 (0.018) & 1.32 (0.061) & \textbf{0.330} (0.015) \\ 
  concrete & \textbf{3.96} (0.24) & 5.21 (0.22) & 4.03 (0.20) \\ 
   \hline
Avg. rank & \textbf{1.33} & 2.44 & 2.22 \\ 
   \hline
\textbf{CRPS} &  &  &  \\ 
  fertility & 0.0986 (0.0065) & 0.103 (0.0046) & \textbf{0.0946} (0.0064) \\ 
  servo & \textbf{0.136} (0.011) & 0.152 (0.0076) & 0.146 (0.013) \\ 
  machine & \textbf{0.195} (0.0081) & 0.208 (0.0072) & 0.224 (0.0072) \\ 
  yacht & \textbf{0.259} (0.027) & 1.28 (0.100) & 0.286 (0.031) \\ 
  autompg & \textbf{1.36} (0.075) & 1.36 (0.076) & 1.59 (0.12) \\ 
  boston & \textbf{1.45} (0.085) & 1.48 (0.060) & 1.87 (0.095) \\ 
  pendulum & \textbf{0.692} (0.041) & 0.701 (0.041) & 1.13 (0.071) \\ 
  energy & 0.205 (0.0076) & 0.749 (0.026) & \textbf{0.172} (0.0051) \\ 
  concrete & \textbf{1.99} (0.084) & 2.75 (0.079) & 2.12 (0.059) \\ 
   \hline
Avg. rank & \textbf{1.22} & 2.56 & 2.22 \\ 
   \hline
\textbf{NLL} &  &  &  \\ 
  fertility & -0.294 (0.087) & -0.269 (0.051) & \textbf{-0.339} (0.097) \\ 
  servo & \textbf{0.142} (0.088) & 0.199 (0.052) & 0.416 (0.21) \\ 
  machine & \textbf{0.436} (0.040) & 0.460 (0.036) & 0.587 (0.063) \\ 
  yacht & \textbf{0.738} (0.11) & 2.30 (0.042) & 1.62 (0.30) \\ 
  autompg & \textbf{2.36} (0.079) & 2.37 (0.073) & 2.52 (0.084) \\ 
  boston & 2.46 (0.074) & \textbf{2.45} (0.070) & 2.86 (0.11) \\ 
  pendulum & \textbf{1.55} (0.063) & 1.56 (0.060) & 2.44 (0.090) \\ 
  energy & \textbf{0.439} (0.052) & 1.77 (0.025) & 0.694 (0.073) \\ 
  concrete & \textbf{2.87} (0.075) & 3.08 (0.043) & 3.14 (0.041) \\ 
   \hline
Avg. rank & \textbf{1.22} & 2.22 & 2.56 \\ 
   \hline
\hline
\end{tabular}
\endgroup
\caption{Average prediction accuracy measures when doing 10-fold cross-validation
  for the UCI benchmark data sets. In parenthesis are standard errors. } 
\label{results_UCI_data}
\end{table}

\section{Conclusion}
We have introduced a novel way of combining boosting with Gaussian process and mixed effects models. This allows for relaxing, first, the zero or linearity assumption for the prior mean function in Gaussian process and mixed effects models in a flexible non-parametric way and, second, the independence assumption made in most boosting algorithms. Further, it can be used as an approach for handling high-cardinality categorical variables in tree-boosting. In simulation experiments and real-world applications, we have shown that this leads to improved prediction accuracy compared to existing state-of-the-art methods. We are currently investigating how our approach can be extended to non-Gaussian data using, e.g., a Laplace approximation \citep{sigrist2021latent}. Further, future research can investigate and compare properties of different Gaussian process approximations for large data applied to the GPBoost algorithm. Future research is also required for computationally efficient uncertainty quantification for the estimated predictor function $F(\cdot)$ and for generating probabilistic predictions that also take into account uncertainty due to $F(\cdot)$. Another direction for future work are convergence results concerning the generalization error as well as finite-sample and asymptotic properties of the estimators obtained from the GPBoost algorithm.

\section*{Acknowledgments}
We thank Hansruedi K\"unsch and three anonymous reviewers for their helpful comments which have helped to improve this article. This research was partially supported by the Swiss Innovation Agency - Innosuisse (grant numbers `25746.1 PFES-ES' and `28408.1 PFES-ES').


\clearpage

\begin{appendices}
\section{Which optimization problems are solved by the MERT/MERF and RE-EM tree algorithms and how are they related to the EM algorithm?}\label{MERFalgo}
The MERT and MERF algorithms \citep{hajjem2011mixed, hajjem2014mixed} and the RE-EM tree \citep{sela2012re, fu2015unbiased} algorithm allow for combining trees and random forests with grouped random effects models. \citet{hajjem2011mixed} and \citet{hajjem2014mixed} claim that the MERT and MERF algorithm are versions of the EM algorithm.\footnote{For instance, \citet{hajjem2011mixed, hajjem2014mixed} state that ``The MERT algorithm is the ML-based EM-algorithm in which we replace the linear structure used to estimate the fixed part of the model by a standard tree structure." and that ``It [the MERT algorithm] is basically an iterative call to a standard RT [regression tree] algorithm within the framework of the expectation–maximization (EM) algorithm".} \citet{sela2012re} call their algorithm RE-EM tree algorithm because it is ``reminiscent" of the EM algorithm, but they acknowledge that it is not a true EM algorithm. Unfortunately, it is not clearly explained in the prior literature which optimization problems these algorithms aim to solve and how they are related to the EM algorithm. In this section, we shed some light on this. 

\subsubsection*{The EM algorithm for linear mixed effects models}
We first briefly describe the EM algorithm \citep{dempster1977maximum} for linear mixed effects models. Using the notation of Section \ref{model_def}, the E-step in iteration $t$ of an EM algorithm for mixed effect models works by, first, finding a maximizer for the predictor function $F$ of the multivariate normal likelihood given the current estimate for the covariance matrix $\Psi_t$:
\begin{equation}\label{estep_F}
\hat F_{t+1}=\argmin_F(y-F)^T{\Psi_t}^{-1}(y-F),
\end{equation}
second, conditional on this obtaining predictions for the random effects $$\hat b_{t+1}=Z\Sigma_tZ^T{\Psi_t}^{-1}(y-\hat F_{t+1}),$$ and then using these two quantities to calculate the expectation of the full data log-likelihood; see, e.g., \citet{laird1982random} and \citet{wu2006nonparametric}. The EM algorithm then proceeds with the M-step by maximizing this expected full data log-likelihood to obtain an estimate for $\Psi_{t+1}$ or its parameters $\theta_{t+1}$.

\subsubsection*{The RE-EM tree algorithm}
The RE-EM tree algorithm of \citet{sela2012re} and \citet{fu2015unbiased} iterates between, first, estimating the structure of a tree using an independent normal likelihood obtained after subtracting predicted values of the random effects from the response variable and, second, jointly estimating the leaf values and covariance parameters using a classical linear mixed effects model. This is not an EM algorithm as it does not involve an E-step that calculates an expectation of a full data log-likelihood. It can, however, be interpreted as a component-wise, or coordinate descent, minimization algorithm that iterates between finding an optimizer for (parts of) $F$, the covariance parameters $\theta$, and the random effects $b$. This can be seen by first noting that, regarding $F$ and $\theta$, the minimization problem
\begin{equation}\label{prob1}
(\hat F, \hat \theta) = \argmin_{(F,\theta)}\frac{1}{2}(y-F)^T{\Psi}^{-1}(y-F)+\frac{1}{2}\log\det\left(\Psi\right)
\end{equation}
is equivalent to
\begin{equation}\label{prob2}
(\hat F, \hat \theta, \hat b) = \argmin_{(F,\theta,b)}\frac{1}{2\sigma^2}(y-F-Zb)^T(y-F-Zb)+\frac{1}{2}b^T\Sigma^{-1}b+\frac{1}{2}\log\det\left(\Psi\right).
\end{equation}
Further, the componentwise minimizer for $b$ given $F$ and $\theta$ corresponds to the best linear unbiased estimator for $b$:
\begin{equation*}
\begin{split}
\hat b =&  (Z^TZ+\sigma^2\Sigma^{-1})^{-1}Z^T(y-F)\\
=& \Sigma Z^T(Z\Sigma Z^T+\sigma^2 I_n)^{-1}(y-F).
\end{split}
\end{equation*}
Using the notation of Section \ref{mult_gauss_loss}, $F(\cdot)=h(\cdot;\alpha)^T\gamma$, where $\alpha$ denotes the splits of a tree and $\gamma$ the leaf values, one iteration of the RE-EM tree algorithm can thus be written as
\begin{equation*}
\begin{split}
\hat\alpha_{t+1} =& \argmin_{\alpha} L\left(F=h(\cdot;\alpha)^T\hat \gamma_t,\hat\theta_t,\hat b_t\right),\\
(\hat\gamma_{t+1},\hat\theta_{t+1}) =& \argmin_{\gamma,\theta} L\left(F=h(\cdot;\hat\alpha_{t+1})^T \gamma,\theta,\hat b_t\right),\\
\hat b_{t+1} =& \argmin_{b} L\left(F=h(\cdot;\hat\alpha_{t+1})^T\hat\gamma_{t+1},\hat\theta_{t+1},b\right),
\end{split}
\end{equation*}
where 
$$L\left(F,\theta,b\right)=\frac{1}{2\sigma^2}(y-F-Zb)^T(y-F-Zb)+\frac{1}{2}b^T\Sigma^{-1}b+\frac{1}{2}\log\det\left(\Psi\right).$$
This corresponds to one iteration of a coordinate descent algorithm applied to the minimization problem in \eqref{prob2} which is equivalent to \eqref{prob1}.

\subsubsection*{The MERT and MERF algorithms}
In contrast to a proper E-step of an EM algorithm for mixed effects model, in the ``E-step" of the MERT and MERF algorithms of \citet{hajjem2011mixed} and \citet{hajjem2014mixed}, the predictor function $\hat F_{t+1}$ is not obtained as maximizer of the multivariate normal likelihood as in \eqref{estep_F}. Instead, the following is done. First, an independent normal likelihood obtained after subtracting predicted values for the random effects $\hat b_{t}$ of the previous iteration from the response variable $y$ is used to learn the predictor function $\hat F_{t+1}$ using a regression tree or a random forest:
$$\hat F_{t+1}=\argmin_F\frac{1}{2}(y-F-\hat b_{t})^T(y-F-\hat b_{t}).$$ 
Second, predictions are calculated for the random effects:
$$\hat b_{t+1}=Z\Sigma_tZ^T{\Psi_t}^{-1}(y-\hat F_{t+1}).$$ 
The M-step is then analogous to a correctly specified EM algorithm. It is unclear to us whether and to which quantities the MERT and MERF algorithms converge as it is nowhere shown that they correspond to correctly specified EM algorithms. 

\section{Proofs for the results involving the Vecchia approximation}\label{proofs}
\begin{proof}[Proof of Proposition \ref{GradVecchia}]
	The negative log-likelihood of the Vecchia approximation in \eqref{vecchia_approx} is given by 
	\begin{equation}\label{ll_vecchia}
	\tilde L(y,F,\theta)=\frac{1}{2}(y-F)^T\tilde {\Psi}^{-1}(y-F)+ \frac{1}{2}\sum_{i=1}^n\log D_i+\frac{n}{2}\log(2\pi).
	\end{equation}
	It follows that
	\begin{equation}\label{grad_vecchia}
	\frac{\partial \tilde L(y,F,\theta)}{\partial \theta_k}= \frac{1}{2\sigma^2}(y-F)^T\frac{\partial}{\partial \theta_k}\tilde {\Psi}^{-1}(y-F)+ \frac{1}{2}\sum_{i=1}^n\frac{1}{D_i}\frac{\partial D_i}{\partial \theta_k}.
	\end{equation}
	Further, we have
	\begin{equation*}
	\frac{\partial}{\partial \theta_k}\tilde {\Psi}^{-1} = \frac{\partial B^T}{\partial \theta_k}{D}^{-1}B+ B^T{D}^{-1}\frac{\partial B}{\partial \theta_k}- B^T{D}^{-1}\frac{\partial D}{\partial \theta_k}{D}^{-1}B.
	\end{equation*}
	We thus obtain the result in the proposition by noting that
	$$(y-F)^T\frac{\partial}{\partial \theta_k}\tilde {\Psi}^{-1}(y-F)=2u_k^Tu - u^T\frac{\partial D}{\partial \theta_k} u,$$
	where $u$ and $u_k$ are given in \eqref{zzk_V}.
\end{proof}
	
\begin{proof}[Proof of Proposition \ref{FIVecchia}]
	We have
	$$\frac{\partial \tilde \Psi}{\partial \theta_k}=-B^{-1}\frac{\partial B}{\partial \theta_k}B^{-1}D B^{-T} - B^{-1}D B^{-T}\frac{\partial B^T}{\partial \theta_k}B^{-T} +B^{-1}\frac{\partial D}{\partial \theta_k} B^{-T}.$$
	It follows that
	$$\tilde {\Psi}^{-1}\frac{\partial \tilde \Psi}{\partial \theta_k}= 
	-B^T{D}^{-1}\frac{\partial B}{\partial \theta_k}B^{-1}D B^{-T} 
	-\frac{\partial B^T}{\partial \theta_k}B^{-T} 
	+B^T{D}^{-1}\frac{\partial D}{\partial \theta_k} B^{-T}.$$
	
	We thus have
	\begin{equation*}
	\begin{split}
	\tilde {\Psi}^{-1}\frac{\partial \tilde \Psi}{\partial \theta_k}\tilde {\Psi}^{-1}\frac{\partial \tilde \Psi}{\partial \theta_l}=& 
	B^T{D}^{-1}\frac{\partial B}{\partial \theta_k}B^{-1}\frac{\partial B}{\partial \theta_l}B^{-1}D B^{-T} 
	+\frac{\partial B^T}{\partial \theta_k} {D}^{-1}\frac{\partial B}{\partial \theta_l}B^{-1}D B^{-T}\\ &-B^T{D}^{-1}\frac{\partial D}{\partial \theta_k} {D}^{-1}\frac{\partial B}{\partial \theta_l}B^{-1}D B^{-T}\\
	&+B^T{D}^{-1}\frac{\partial B}{\partial \theta_k}B^{-1}D B^{-T} \frac{\partial B^T}{\partial \theta_l}B^{-T} 
	+\frac{\partial B^T}{\partial \theta_k}B^{-T} \frac{\partial B^T}{\partial \theta_l}B^{-T}\\
	&-B^T{D}^{-1}\frac{\partial D}{\partial \theta_k} B^{-T}\frac{\partial B^T}{\partial \theta_l}B^{-T}\\
	&-B^T{D}^{-1}\frac{\partial B}{\partial \theta_k}B^{-1}\frac{\partial D}{\partial \theta_l} B^{-T}
	-\frac{\partial B^T}{\partial \theta_k}{D}^{-1}\frac{\partial D}{\partial \theta_l} B^{-T}\\
	&+B^T{D}^{-1}\frac{\partial D}{\partial \theta_k} {D}^{-1}\frac{\partial D}{\partial \theta_l} B^{-T}.
	\end{split}
	\end{equation*}
	Due to the cyclicality of the trace, it follows that
	\begin{equation*}
	\begin{split}
	\text{tr}\left(\tilde {\Psi}^{-1}\frac{\partial \tilde \Psi}{\partial \theta_k}\tilde {\Psi}^{-1}\frac{\partial \tilde \Psi}{\partial \theta_l}\right)=& 
	\text{tr}\left(\frac{\partial B}{\partial \theta_k}B^{-1}\frac{\partial B}{\partial \theta_l}B^{-1}\right)
	+\text{tr}\left(\frac{\partial B^T}{\partial \theta_k} {D}^{-1}\frac{\partial B}{\partial \theta_l}B^{-1}D B^{-T}\right)\\
	&-\text{tr}\left(\frac{\partial D}{\partial \theta_k} {D}^{-1}\frac{\partial B}{\partial \theta_l}B^{-1}\right)\\
	&+\text{tr}\left({D}^{-1}\frac{\partial B}{\partial \theta_k}B^{-1}D B^{-T} \frac{\partial B^T}{\partial \theta_l}\right)
	+\text{tr}\left(\frac{\partial B^T}{\partial \theta_k}B^{-T} \frac{\partial B^T}{\partial \theta_l}B^{-T}\right)\\
	&-\text{tr}\left({D}^{-1}\frac{\partial D}{\partial \theta_k} B^{-T}\frac{\partial B^T}{\partial \theta_l}\right)\\
	&-\text{tr}\left({D}^{-1}\frac{\partial B}{\partial \theta_k}B^{-1}\frac{\partial D}{\partial \theta_l}\right)
	-\text{tr}\left(\frac{\partial B^T}{\partial \theta_k}{D}^{-1}\frac{\partial D}{\partial \theta_l} B^{-T}\right)\\
	&+\text{tr}\left({D}^{-1}\frac{\partial D}{\partial \theta_k} {D}^{-1}\frac{\partial D}{\partial \theta_l}\right).
	\end{split}
	\end{equation*}
	Since $\frac{\partial B}{\partial \theta_k}$ is lower triangular with zeros on the diagonal, we obtain
	\begin{equation*}
	\begin{split}
	\text{tr}\left(\tilde {\Psi}^{-1}\frac{\partial \tilde \Psi}{\partial \theta_k}\tilde {\Psi}^{-1}\frac{\partial \tilde \Psi}{\partial \theta_l}\right)=& 
	\text{tr}\left(\frac{\partial B^T}{\partial \theta_k} {D}^{-1}\frac{\partial B}{\partial \theta_l}B^{-1}D B^{-T}\right)+\text{tr}\left({D}^{-1}\frac{\partial B}{\partial \theta_k}B^{-1}D B^{-T} \frac{\partial B^T}{\partial \theta_l}\right)\\
	&+\text{tr}\left({D}^{-1}\frac{\partial D}{\partial \theta_k} {D}^{-1}\frac{\partial D}{\partial \theta_l}\right)\\
	=& 
	2\text{tr}\left(B^{-T}\frac{\partial B^T}{\partial \theta_k} {D}^{-1}\frac{\partial B}{\partial \theta_l}B^{-1}D \right)+\text{tr}\left({D}^{-1}\frac{\partial D}{\partial \theta_k} {D}^{-1}\frac{\partial D}{\partial \theta_l}\right),
	\end{split}
	\end{equation*}
	and the statement in \eqref{FIVecciaentry} follows.
\end{proof}

\begin{proof}[Proof of Proposition \ref{PredVEcchiaOF}]
	We have 
	$$\begin{pmatrix} B & 0 \\ B_{po}&B_p\end{pmatrix}^T
	\begin{pmatrix} {D}^{-1}& 0 \\ 0&{D_p}^{-1}\end{pmatrix}
	\begin{pmatrix} B & 0 \\ B_{po}&B_p\end{pmatrix}
	=\begin{pmatrix} B_o^T{D_o}^{-1}B_o + B_{po}^T{D_p}^{-1}B_{po} & B_{po}^T{D_p}^{-1}B_p \\ B_p^T{D_p}^{-1}B_{po}&B_p^T{D_p}^{-1}B_p\end{pmatrix}.
	$$
	Since $$\left(B_p^T{D_p}^{-1}B_p\right)^{-1}B_p^T{D_p}^{-1}B_{po}=B_p^{-1}B_{po},$$ the result follows from Theorem 12.2 in \citet{rue2010discrete}.
\end{proof}

\begin{proof}[Proof of Proposition \ref{PredVEcchiaPF}]
	We have 
	$$\begin{pmatrix} B_p & 0 \\ B_{op}&B_o\end{pmatrix}^T
	\begin{pmatrix} {D_p}^{-1} & 0 \\0&{D_o}^{-1}\end{pmatrix} 
	\begin{pmatrix} B_p & 0 \\ B_{op}&B_o\end{pmatrix}
	=\begin{pmatrix} B_p^T{D_p}^{-1}B_p + B_{op}^T{D_o}^{-1}B_{op} & B_{op}^T{D_o}^{-1}B_o \\ B_o^T{D_o}^{-1}B_{op}&B_o^T{D_o}^{-1}B_o\end{pmatrix},
	$$ 
	and the result follows from Theorem 12.2 in \citet{rue2010discrete}.
\end{proof}

\clearpage

\section{Additional results for the simulated experiments}\label{sim_res_oth}
In the following, we present the results for the simulated experiments in Section \ref{simul} for the two predictor functions which are not reported in the body of the article. 

\begin{table}[ht!]
\centering
\begingroup\footnotesize
\scalebox{0.9}{
\begin{tabular}{rlllllll|l}
  \hline
\hline
  & GPBoost & LinearME & LSBoost & CatBoost & mboost & MERF & REEMtree & GPBOOS \\ 
  \hline
RMSE & \bf{1.065} & 1.237 & 1.112 & 1.141 & 1.239 & 1.088 & 1.119 & 1.065 \\ 
  (SE) & (0.00118) & (0.00123) & (0.00154) & (0.00134) & (0.00167) & (0.00118) & (0.00131) & (0.00117) \\ 
  \lbrack p-val\rbrack &  & [1.3e-120] & [9.65e-61] & [1.01e-90] & [2.88e-106] & [3.33e-82] & [3.75e-86] & [0.0624] \\ 
   \hline
RMSE\_new & 1.426 & 1.549 & 1.446 & \bf{1.425} & 1.429 & 1.442 & 1.466 & 1.426 \\ 
  (SE) & (0.00267) & (0.00273) & (0.00286) & (0.0027) & (0.00266) & (0.00262) & (0.00266) & (0.00265) \\ 
  \lbrack p-val\rbrack &  & [3.51e-112] & [5e-44] & [0.00398] & [4.22e-13] & [9.94e-62] & [1.25e-78] & [0.905] \\ 
   \hline
RMSE\_F & \bf{0.2073} & 0.6435 &  &  & 0.2280 & 0.2962 & 0.3916 & 0.2082 \\ 
  (SE) & (0.00129) & (0.00125) &  &  & (0.00161) & (0.000934) & (0.00184) & (0.00128) \\ 
  \lbrack p-val\rbrack &  & [1.64e-144] &  &  & [1.7e-32] & [6.05e-98] & [1.12e-103] & [0.0199] \\ 
  RMSE\_b & \bf{0.3111} & 0.3542 &  &  & 0.6986 & 0.3179 & 0.3242 & 0.3112 \\ 
  (SE) & (0.00113) & (0.00125) &  &  & (0.00229) & (0.00108) & (0.00113) & (0.00111) \\ 
  \lbrack p-val\rbrack &  & [9.74e-76] &  &  & [9.43e-122] & [2.67e-23] & [1.01e-54] & [0.238] \\ 
   \hline
RMSE $\sigma^2_1$ & 0.07503 & 0.08060 &  &  &  & 0.07495 & 0.07532 & 0.07527 \\ 
  Bias $\sigma^2_1$ & 0.003267 & 0.005301 &  &  &  & 0.004099 & 0.0008061 & 0.004154 \\ 
  RMSE $\sigma^2$ & 0.1177 & 0.4077 & 0.08416 & 0.09373 & 0.3652 & 0.08100 & 0.06141 & 0.05531 \\ 
  Bias $\sigma^2$ & -0.1144 & 0.4067 & -0.07411 & 0.09006 & 0.3628 & 0.07779 & 0.05595 & 0.05079 \\ 
   \hline
Time (s) & 0.5760 & 0.02151 & 0.2079 & 5.754 & 9.129 & 367.4 & 0.8807 & 2.539 \\ 
   \hline
\hline
\end{tabular}
}
\endgroup
\caption{Results of the simulated experiments for the grouped data 
                          and the predictor function F = `friedman3'. For the prediction accuracy metrics for $y$, $F$, and $b$, averages over 
                      the simulation runs are reported. Corresponding standard errors are in parentheses.
                      P-values are calculated using paired t-tests comparing the GPBoost algorithm to the other approaches.
                      `GBPOOS' refers to the GPBoostOOS algorithm. Results for the test data with new groups are denoted by `\_new'. The smallest values are in boldface (excluding `GPBOOS'). 
                      An empty value indicates that the required predictions or estimates cannot be calculated.
                      Time refers to the average wall-clock time in seconds.} 
\label{results_friedman3_one_way}
\end{table}

\begin{table}[ht!]
\centering
\begingroup\footnotesize
\scalebox{0.9}{
\begin{tabular}{rlllllll|l}
  \hline
\hline
  & GPBoost & LinearME & LSBoost & CatBoost & mboost & MERF & REEMtree & GPBOOS \\ 
  \hline
RMSE & 1.049 & \bf{1.044} & 1.074 & 1.139 & 1.225 & 1.124 & 1.087 & 1.049 \\ 
  (SE) & (0.00107) & (0.00105) & (0.00159) & (0.00145) & (0.00177) & (0.00127) & (0.00122) & (0.00108) \\ 
  \lbrack p-val\rbrack &  & [4.2e-48] & [1.88e-43] & [7.1e-96] & [4.56e-112] & [2.06e-111] & [1.43e-82] & [5.26e-08] \\ 
   \hline
RMSE\_new & 1.418 & \bf{1.415} & 1.425 & 1.421 & 1.424 & 1.471 & 1.445 & 1.419 \\ 
  (SE) & (0.00256) & (0.00256) & (0.00257) & (0.00255) & (0.0026) & (0.00254) & (0.00262) & (0.00256) \\ 
  \lbrack p-val\rbrack &  & [8.24e-40] & [2.49e-30] & [1.66e-08] & [1.41e-42] & [1.68e-99] & [2.68e-76] & [3.06e-05] \\ 
   \hline
RMSE\_F & 0.1077 & \bf{0.04600} &  &  & 0.1654 & 0.4015 & 0.2954 & 0.1111 \\ 
  (SE) & (0.0014) & (0.00253) &  &  & (0.00125) & (0.000874) & (0.00151) & (0.00145) \\ 
  \lbrack p-val\rbrack &  & [7.6e-63] &  &  & [4.52e-76] & [2.11e-137] & [7.61e-106] & [9.73e-08] \\ 
  RMSE\_b & 0.3068 & \bf{0.3055} &  &  & 0.6920 & 0.3283 & 0.3173 & 0.3069 \\ 
  (SE) & (0.000974) & (0.000972) &  &  & (0.00215) & (0.000933) & (0.00103) & (0.000975) \\ 
  \lbrack p-val\rbrack &  & [1.36e-18] &  &  & [5.84e-122] & [1.43e-60] & [5.33e-49] & [7.62e-05] \\ 
   \hline
RMSE $\sigma^2_1$ & 0.06990 & 0.06986 &  &  &  & 0.07142 & 0.07124 & 0.06998 \\ 
  Bias $\sigma^2_1$ & -0.01381 & -0.01199 &  &  &  & -0.01348 & -0.01570 & -0.01279 \\ 
  RMSE $\sigma^2$ & 0.03107 & 0.02032 & 0.04995 & 0.1274 & 0.3616 & 0.1322 & 0.03238 & 0.02252 \\ 
  Bias $\sigma^2$ & -0.02395 & -0.004143 & -0.03722 & 0.1242 & 0.3595 & 0.1289 & 0.02331 & 0.009169 \\ 
   \hline
Time (s) & 0.6096 & 0.01585 & 0.1182 & 3.278 & 8.070 & 323.7 & 0.5651 & 2.656 \\ 
   \hline
\hline
\end{tabular}
}
\endgroup
\caption{Results of the simulated experiments for the grouped data 
                          and the predictor function F = `linear'. For the prediction accuracy metrics for $y$, $F$, and $b$, averages over 
                      the simulation runs are reported. Corresponding standard errors are in parentheses.
                      P-values are calculated using paired t-tests comparing the GPBoost algorithm to the other approaches.
                      `GBPOOS' refers to the GPBoostOOS algorithm. Results for the test data with new groups are denoted by `\_new'. The smallest values are in boldface (excluding `GPBOOS'). 
                      An empty value indicates that the required predictions or estimates cannot be calculated.
                      Time refers to the average wall-clock time in seconds.} 
\label{results_linear_one_way}
\end{table}

\begin{table}[ht!]
\centering
\begingroup\footnotesize
\scalebox{0.9}{
\begin{tabular}{rlllll|l}
  \hline
\hline
  & GPBoost & LinearGP & LSBoost & mboost & TwoStep & GPBOOS \\ 
  \hline
RMSE & \bf{1.271} & 1.361 & 1.379 & 1.331 & 1.287 & 1.273 \\ 
  (SE) & (0.00456) & (0.00511) & (0.00515) & (0.00519) & (0.00475) & (0.00468) \\ 
  \lbrack p-val\rbrack &  & [2.14e-45] & [2.6e-46] & [1.38e-24] & [1.21e-07] & [0.188] \\ 
  CRPS & \bf{0.7225} & 0.7644 & 0.8092 & 0.7530 & 0.7300 & 0.7174 \\ 
  (SE) & (0.00291) & (0.0028) & (0.00415) & (0.00299) & (0.00304) & (0.0026) \\ 
  \lbrack p-val\rbrack &  & [7.5e-35] & [8.07e-46] & [4.01e-20] & [7.2e-05] & [2.96e-05] \\ 
   \hline
RMSE\_ext & \bf{1.437} & 1.511 & 1.556 & 1.743 & 1.450 & 1.440 \\ 
  (SE) & (0.0107) & (0.0101) & (0.0169) & (0.0412) & (0.0109) & (0.0107) \\ 
  \lbrack p-val\rbrack &  & [7.57e-35] & [3.05e-16] & [1.31e-11] & [2.06e-06] & [0.0808] \\ 
  CRPS\_ext & \bf{0.8158} & 0.8522 & 0.9358 & 1.032 & 0.8259 & 0.8138 \\ 
  (SE) & (0.00652) & (0.00581) & (0.013) & (0.0299) & (0.00703) & (0.00619) \\ 
  \lbrack p-val\rbrack &  & [4.31e-27] & [5.49e-21] & [3.83e-11] & [9.96e-07] & [0.0648] \\ 
   \hline
RMSE\_sum & \bf{11.60} & 11.87 & 14.16 & 17.50 & 11.76 & 11.63 \\ 
  (SE) & (0.208) & (0.195) & (0.319) & (0.726) & (0.21) & (0.208) \\ 
  \lbrack p-val\rbrack &  & [1.35e-05] & [2.06e-18] & [1.27e-12] & [0.00157] & [0.227] \\ 
  CRPS\_sum & \bf{6.342} & 6.480 & 8.604 & 10.21 & 6.495 & 6.330 \\ 
  (SE) & (0.105) & (0.098) & (0.219) & (0.444) & (0.116) & (0.103) \\ 
  \lbrack p-val\rbrack &  & [9.77e-05] & [3.63e-24] & [4.58e-14] & [0.00011] & [0.438] \\ 
   \hline
RMSE\_F & \bf{0.5061} & 0.6870 &  & 0.5360 & 0.5360 & 0.5096 \\ 
  (SE) & (0.007) & (0.00571) &  & (0.00732) & (0.00732) & (0.00719) \\ 
  \lbrack p-val\rbrack &  & [1.98e-54] &  & [7.62e-10] & [7.62e-10] & [0.279] \\ 
  RMSE\_b & \bf{0.6690} & 0.6790 &  & 0.6775 & 0.6775 & 0.6676 \\ 
  (SE) & (0.00458) & (0.00531) &  & (0.00504) & (0.00504) & (0.00462) \\ 
  \lbrack p-val\rbrack &  & [0.000699] &  & [4.82e-05] & [4.82e-05] & [0.212] \\ 
   \hline
RMSE $\sigma^2_1$ & 0.2480 & 0.2763 &  &  & 0.3507 & 0.2473 \\ 
  Bias $\sigma^2_1$ & -0.006707 & -0.001842 &  &  & -0.2370 & 0.0006891 \\ 
  RMSE $\rho$ & 0.02832 & 0.02810 &  &  & 0.02905 & 0.02805 \\ 
  Bias $\rho$ & -0.005450 & -0.004948 &  &  & -0.004900 & -0.006572 \\ 
  RMSE $\sigma^2$ & 0.3187 & 0.3951 & 0.4018 & 0.4257 & 0.1943 & 0.2367 \\ 
  Bias $\sigma^2$ & -0.2290 & 0.3643 & -0.1137 & 0.3940 & -0.08988 & 0.1921 \\ 
   \hline
Time (s) & 21.60 & 0.6269 & 0.09670 & 2.857 & 0.5620 & 32.14 \\ 
   \hline
\hline
\end{tabular}
}
\endgroup
\caption{Results of the simulated experiments for the spatial data 
                          and the predictor function F = `friedman3'. For the prediction accuracy metrics for $y$, $F$, and $b$, averages over 
                      the simulation runs are reported. Corresponding standard errors are in parentheses.
                      P-values are calculated using paired t-tests comparing the GPBoost algorithm to the other approaches.
                      `GBPOOS' refers to the GPBoostOOS algorithm. Results for the ``extrapolation" test data are denoted by `\_ext', and results for the predictions of sums are denoted by `\_sum'. The smallest values are in boldface (excluding `GPBOOS'). 
                      An empty value indicates that the required predictions or estimates cannot be calculated.
                      Time refers to the average wall-clock time in seconds.} 
\label{results_friedman3_spatial}
\end{table}

\begin{table}[ht!]
\centering
\begingroup\footnotesize
\scalebox{0.9}{
\begin{tabular}{rlllll|l}
  \hline
\hline
  & GPBoost & LinearGP & LSBoost & mboost & TwoStep & GPBOOS \\ 
  \hline
RMSE & 1.205 & \bf{1.183} & 1.328 & 1.276 & 1.214 & 1.208 \\ 
  (SE) & (0.00422) & (0.0039) & (0.00594) & (0.00505) & (0.00464) & (0.00441) \\ 
  \lbrack p-val\rbrack &  & [4.69e-30] & [1.79e-52] & [3.41e-47] & [1.95e-06] & [0.00153] \\ 
  CRPS & 0.6807 & \bf{0.6677} & 0.7632 & 0.7222 & 0.6861 & 0.6819 \\ 
  (SE) & (0.00246) & (0.00219) & (0.00431) & (0.00296) & (0.00277) & (0.00248) \\ 
  \lbrack p-val\rbrack &  & [2.5e-29] & [2.73e-44] & [1.75e-46] & [4.65e-05] & [0.0474] \\ 
   \hline
RMSE\_ext & 1.392 & \bf{1.373} & 1.518 & 1.700 & 1.400 & 1.395 \\ 
  (SE) & (0.00979) & (0.00965) & (0.0155) & (0.0403) & (0.00981) & (0.01) \\ 
  \lbrack p-val\rbrack &  & [8.13e-20] & [2.03e-13] & [4.93e-12] & [3.78e-05] & [0.00458] \\ 
  CRPS\_ext & 0.7879 & \bf{0.7764} & 0.8871 & 1.004 & 0.7924 & 0.7890 \\ 
  (SE) & (0.0058) & (0.00562) & (0.0112) & (0.0289) & (0.00584) & (0.00578) \\ 
  \lbrack p-val\rbrack &  & [1.82e-19] & [3.68e-15] & [1.26e-11] & [6.05e-05] & [0.0728] \\ 
   \hline
RMSE\_sum & 11.44 & \bf{11.32} & 14.15 & 17.29 & 11.51 & 11.49 \\ 
  (SE) & (0.198) & (0.196) & (0.287) & (0.702) & (0.198) & (0.201) \\ 
  \lbrack p-val\rbrack &  & [4.8e-05] & [1.78e-16] & [1.11e-13] & [0.0177] & [0.00526] \\ 
  CRPS\_sum & 6.152 & \bf{6.070} & 8.339 & 9.928 & 6.200 & 6.177 \\ 
  (SE) & (0.0969) & (0.0952) & (0.183) & (0.42) & (0.0986) & (0.0979) \\ 
  \lbrack p-val\rbrack &  & [2.88e-07] & [9.54e-23] & [3.69e-15] & [0.00553] & [0.00862] \\ 
   \hline
RMSE\_F & 0.3265 & \bf{0.2099} &  & 0.3568 & 0.3568 & 0.3382 \\ 
  (SE) & (0.0102) & (0.013) &  & (0.0111) & (0.0111) & (0.0106) \\ 
  \lbrack p-val\rbrack &  & [1.32e-32] &  & [6.59e-11] & [6.59e-11] & [4.45e-05] \\ 
  RMSE\_b & 0.6662 & \bf{0.6600} &  & 0.6721 & 0.6721 & 0.6673 \\ 
  (SE) & (0.0059) & (0.00554) &  & (0.0064) & (0.0064) & (0.00607) \\ 
  \lbrack p-val\rbrack &  & [0.0135] &  & [4e-05] & [4e-05] & [0.0592] \\ 
   \hline
RMSE $\sigma^2_1$ & 0.2279 & 0.2265 &  &  & 0.2483 & 0.2275 \\ 
  Bias $\sigma^2_1$ & -0.03959 & -0.03913 &  &  & -0.1413 & -0.03173 \\ 
  RMSE $\rho$ & 0.03602 & 0.03598 &  &  & 0.03643 & 0.03506 \\ 
  Bias $\rho$ & -0.007792 & -0.007953 &  &  & -0.008058 & -0.009017 \\ 
  RMSE $\sigma^2$ & 0.1644 & 0.1123 & 0.4185 & 0.3645 & 0.1201 & 0.1347 \\ 
  Bias $\sigma^2$ & -0.1222 & -0.02047 & 0.1885 & 0.3444 & -0.05652 & 0.06216 \\ 
   \hline
Time (s) & 26.46 & 0.6568 & 0.06625 & 1.989 & 0.5880 & 43.50 \\ 
   \hline
\hline
\end{tabular}
}
\endgroup
\caption{Results of the simulated experiments for the spatial data 
                          and the predictor function F = `linear'. For the prediction accuracy metrics for $y$, $F$, and $b$, averages over 
                      the simulation runs are reported. Corresponding standard errors are in parentheses.
                      P-values are calculated using paired t-tests comparing the GPBoost algorithm to the other approaches.
                      `GBPOOS' refers to the GPBoostOOS algorithm. Results for the ``extrapolation" test data are denoted by `\_ext', and results for the predictions of sums are denoted by `\_sum'. The smallest values are in boldface (excluding `GPBOOS'). 
                      An empty value indicates that the required predictions or estimates cannot be calculated.
                      Time refers to the average wall-clock time in seconds.} 
\label{results_linear_spatial}
\end{table}

\end{appendices}

\clearpage
	
\bibliographystyle{abbrvnat}
\bibliography{bib_gpboost}

\end{document}